\documentclass[11pt]{article}

\usepackage[letterpaper,margin=1in]{geometry}
\usepackage[utf8]{inputenc}
\usepackage[T1]{fontenc}
\usepackage{lmodern}
\usepackage[numbers,sort&compress]{natbib}
\usepackage{amsmath,amssymb,amsfonts,amsthm,mathtools}
\usepackage{booktabs}
\usepackage{array}
\usepackage{multirow}
\newtheorem{theorem}{Theorem}

\usepackage{graphicx}
\usepackage{subcaption}
\usepackage{float}
\graphicspath{{figures/}}
\usepackage{microtype}
\usepackage{hyperref}
\usepackage{url}

\title{SMILE: Self-Explainable Multimodal Information Bottleneck for Medical Diagnosis}
\author{Yuqing YANG\textsuperscript{1}, Alexander SCHMATZ\textsuperscript{2,6}, Zhaozhao MA\textsuperscript{3}, Changkyu CHOI\textsuperscript{4}, \\ Robert JENSSEN\textsuperscript{5}, 
Shujian YU\textsuperscript{5,6,*}}
\newcommand{\authoraffiliations}{%
  \textsuperscript{1}Independent Researcher\\
  \textsuperscript{2}Leiden Institute of Advanced Computer Science (LIACS), Leiden University, Leiden, The Netherlands\\
  \textsuperscript{3}School of Applied and Creative Computing, Purdue University, West Lafayette, IN, USA\\
  \textsuperscript{4}Department of Informatics, University of Oslo, Oslo, Norway\\
  \textsuperscript{5}Machine Learning Group, UiT --- The Arctic University of Norway, Troms\o, Norway\\
  \textsuperscript{6}Quantitative Data Analytics Group, Vrije Universiteit Amsterdam, Amsterdam, The Netherlands\\[0.35em]
  \textsuperscript{*}\emph{Corresponding author: yusj9011@gmail.com.}
}
\date{}

\makeatletter
\renewcommand{\maketitle}{%
  \begin{center}
    {\Large\bfseries \@title\par}
    \vspace{0.9em}
    {\small\itshape \@author\par}
    \vspace{0.65em}
    {\footnotesize \authoraffiliations\par}
  \end{center}
  \vspace{1.25em}
}
\makeatother

\theoremstyle{plain}
\newtheorem{proposition}{Proposition}

\theoremstyle{remark}

\begin{document}
\maketitle

\begin{abstract}
Explainability is increasingly seen as a crucial requirement in AI-based medical diagnosis, particularly in safety-critical clinical decision-making. Most existing explainability methods in healthcare operate in a post-hoc manner and are predominantly designed for unimodal data, which limits their applicability in increasingly prevalent multimodal diagnostic settings. 
This paper addresses the problem of self-explainable multimodal diagnosis by formulating it within the information bottleneck (IB) framework. We propose a unified learning paradigm that jointly optimizes predictive performance and modality-specific explainability by identifying the most informative elements inside each modality that contribute to diagnostic decisions. To enable tractable and stable optimization, we employ a matrix-based R\'enyi’s $\alpha$-order entropy functional under the assumption of sufficiently expressive encoders. Extensive experiments on representative medical datasets spanning heterogeneous modalities demonstrate that the proposed method consistently achieves strong diagnostic performance, including an absolute accuracy improvement of 9.1 percentage points on the iCTCF dataset. Moreover, the learned explanations provide transparent and modality-aware insights into feature relevance, thereby improving both the explainability and generalization.
\end{abstract}

\section{Introduction}
Deep learning has achieved substantial success in a wide range of application domains, including medical diagnosis, where data-driven models have demonstrated strong predictive performance across diverse clinical tasks~\cite{holzinger2019causability, he2022active}. Despite these advances, the increasing complexity of modern deep models has rendered their decision-making processes largely opaque; their ``black-box" nature becomes a significant barrier to clinical adoption. In safety-critical medical settings, such opacity undermines clinicians’ confidence in assessing the reliability of automated diagnostic decisions, thereby motivating a growing demand for explainable artificial intelligence (XAI)~\cite{jia2020clinical,samek2019explainable}.


Representative post-hoc explanation methods explain the predictions of an already-trained model without modifying its training process. Methods such as Grad-CAM~\cite{selvaraju2017grad}, LIME~\cite{ribeiro2016should}, and SHAP~\cite{lundberg2017unified}, aim to identify salient input features or regions that influence model outputs, and have been applied to explain medical images~\cite{gallo2023functional} as well as non-image clinical data~\cite{wang2021interpretability}. While effective in unimodal scenarios, these methods are predominantly designed for natural images or text, and their post-hoc nature may limit the faithfulness of the resulting explanations. As a result, their applicability to complex clinical decision-making pipelines remains limited.

In real-world medical diagnosis, clinical decisions are rarely based on a single data source. Instead, heterogeneous modalities—such as medical imaging, genomic profiles, and structured clinical records—are jointly considered, with each modality providing unique and complementary information about patient health. Multimodal learning frameworks that integrate such heterogeneous data have demonstrated notable improvements in diagnostic performance by capturing cross-modal dependencies and shared representations~\cite{wang2021mogonet, han2022multimodal, fang2024dynamic}.

Despite their success, multimodal diagnostic models introduce additional challenges for explainability. The heterogeneity of data types and the complex interactions between modalities make it difficult to interpret how individual features and modalities contribute to the final prediction. Existing explainability approaches for multimodal models largely rely on post-hoc extensions of unimodal methods~\cite{kim2018interpretability, wang2021interpretability}, which may not faithfully reflect the model's decision-making process. Consequently, there remains a lack of principled frameworks that jointly optimize multimodal prediction and explanation within a unified learning objective.


To fill this gap, we study self-explainable multimodal medical diagnosis from an information-theoretic perspective by formulating it as an information bottleneck (IB) problem~\cite{tishby99information}, where modality-specific explainers are learned to identify compact subsets of input elements that retain diagnostic information while discarding redundant information. This formulation naturally couples explanation generation with representation learning, ensuring that the selected features are intrinsically aligned with the model's decision process rather than obtained post hoc.
Moreover, distinct explainers are designed for each modality to reflect their heterogeneous clinical data characteristics. To enable tractable and stable optimization of the IB objective, we assume sufficiently expressive modality-specific encoders~\cite{tian2020makes} and employ a matrix-based R\'enyi’s $\alpha$-order entropy functional~\cite{giraldo2014measures,yu2019multivariate} for mutual information estimation, which avoids introducing auxiliary parametric models such as the mutual information neural estimator (MINE)~\cite{belghazi2018mutual}. 


The main contributions of this work include:
\begin{itemize}
\item We propose \textbf{S}elf-explainable \textbf{M}ultimodal \textbf{I}nformation bott\textbf{LE}neck (SMILE), which, to the best of our knowledge, is the first IB-based self-explainable framework for multimodal medical diagnosis. Different from post-hoc approaches that analyze an already-trained predictor, SMILE learns the predictive representation and explanation jointly, making the selected features an intrinsic component of the decision process.

\item We develop a relaxed and tractable training objective for SMILE that enables efficient and stable learning. Furthermore, we provide theoretical insights into the generalization behavior of SMILE, showing that it does not suffer from an intrinsic trade-off between explainability and predictive accuracy.

\item We evaluate SMILE on multiple representative multimodal medical datasets, where it consistently achieves state-of-the-art diagnostic performance. In addition, SMILE produces instance-wise, modality-specific explanations that are quantitatively and qualitatively validated, demonstrating improved faithfulness and explainability.
\end{itemize}

The remainder of this paper is structured as follows: Section~\ref{sec: rw} gives a brief review of related works, and Section~\ref{sec: method} introduces each module of SMILE and gives a generalization analysis of the model. Section~\ref{sec: ex} presents both the quantitative and qualitative experimental results on diverse representative datasets. Finally, Section~\ref{sec: conclude} draws the conclusion.

\section{Related Work}
\label{sec: rw}

\subsection{Explainability for Medical Diagnosis}

Explainable AI (XAI) has been extensively studied to improve the transparency of AI-based diagnostic systems. A large body of work focuses on \emph{post-hoc} explanation, which can be broadly grouped into saliency/attribution-based and perturbation-based approaches. Saliency-based methods are widely applied in medical imaging data to visualize important areas of interest, such as tumors or lesions, using methods like Grad-CAM~\cite{selvaraju2017grad} and SmoothGrad~\cite{smilkov2017smoothgrad}. Perturbation-based methods estimate feature importance by measuring changes in model outputs under controlled input modifications; classical formulations include influence functions~\cite{koh2017understanding} and additive feature attribution such as SHAP~\cite{lundberg2017unified}.

Beyond generic post-hoc explanations, recent studies have begun to emphasize whether explanations are meaningful and relevant to clinical decision-making processes. For instance, guideline-oriented frameworks incorporate medical guidelines into explanations to improve clinical relevance~\cite{zhu2019guideline}, and clinician-centered principles provide practical recommendations for clinical XAI design~\cite{jina2022clinicalxai}. In parallel, existing work has explored explainability methods tailored to specific data modalities, including structured electronic health records and clinical time series (e.g., Dynamask~\cite{crabbe2021explaining}) and graph-structured biomedical data (e.g., GNNExplainer~\cite{ying2019gnnexplainer}).

With the growing availability of heterogeneous clinical data, explainability has become an increasingly important consideration for \emph{multimodal} diagnostic models.  Existing explanation methods for multimodal medical AI are predominantly post-hoc, often rely on classic approaches, and tend to focus on specific disease types, limiting their broader applicability. For example, \cite{kraaijveld2022multi} extends the concept activation method~\cite{kim2018interpretability} to generate human-understandable concepts for explaining prostate cancer detection using PET and CT scans. In another study~\cite{wang2021interpretability}, image data and metadata are combined for skin lesion diagnosis, where Grad-CAM is applied to interpret image features, and kernel SHAP~\cite{lundberg2017unified} is used to explain metadata. More recently, clinical perspectives highlight that diagnostic decisions often rely on multiple data modalities, which makes it more challenging to present and integrate explanations across modalities~\cite{pahud2024orchestrator}.

A closely related line of work studies multimodal explainability at the \emph{modality} or \emph{interaction} level, quantifying which modalities (or modality interactions) drive a prediction. Representative examples include interpretable fusion mechanisms that assign relevance scores to modalities and their interactions (e.g., InTense~\cite{varshneya2024intense}) and methods that explicitly quantify cross-modal interactions and modality contributions (e.g., InterSHAP~\cite{wenderoth2025intershap}). 
These modality-level analyses provide coarse-grained attribution over modalities; clinical interpretation often requires identifying \emph{which input elements within each modality} are diagnostically informative for a given patient.

In contrast to most existing multimodal explanation methods in medical AI, our work targets \emph{feature-level} self-explainability in a modality-aware manner, providing a generic framework that can, in principle, be applied to diverse data modalities and potentially adapted to multiple disease types.

\subsection{Information Bottleneck (IB) in Deep Learning}


The IB principle~\cite{tishby99information,shwartz2017opening} considers extracting relevant information from a random variable $X$ to predict another variable $Y$. It operates by identifying a ``bottleneck" variable $\tilde{X}$ that maximizes its predictive power to $Y$, as expressed by mutual information $I(Y;\tilde{X})$, while restricting the amount of information it carries about $X$, formulated as $I(X;\tilde{X})$:
\begin{equation}\label{eq:IB_original}
    \min_{p(\tilde{X}|X)} -I(Y;\tilde{X}) + \beta I(X;\tilde{X}),
\end{equation}
where $\beta \geq 0$ is a Lagrange multiplier. 

IB-inspired objectives have been studied to encourage compact, generalizable representations in deep learning~\cite{alemi2016deep}, in which $\tilde{X}$ usually refers to the latent layer representations.  


Recently, IB has been extended to multimodal representation learning to address redundancy and modality-specific noise~\cite{song2021multicolor, mai2022multimodal,fang2024dynamic}. For example, multimodal information bottleneck formulations aim to learn minimal and sufficient unimodal and joint representations under information-theoretic regularization~\cite{mai2022multimodal}. More recent work further investigates how to set regularization weights or handle imbalanced task-relevant information across modalities (e.g., OMIB)~\cite{wu2025omib}. While these works have demonstrated strong empirical generalization performance, they lack theoretical justification and the ability to explain feature contributions to the decision.


On the other hand, the IB principle has also been explored for single-modality explainability. Its central idea is to learn a differentiable mask \( M \) over the input \( X \), such that the masked input \( X \odot M \) is maximally informative for the decision. Owing to its flexibility, this formulation has been adopted for both post-hoc explainability~\cite{bang2021explaining} and self-explainability~\cite{yu2021graph,zheng2024brainib}. However, extending these approaches to heterogeneous multimodal clinical data remains an open challenge.

Closer to our goal, recent studies explicitly leverage IB to explain multimodal representations (e.g., IB-based attribution for image--text representations~\cite{wang2023vl_ib}, and cross-attention-guided multimodal IB attribution~\cite{bourigault2024cam2ib}).
However, these methods are designed as post-hoc attribution tools for specific (mostly vision-language) architectures, and do not address self-explainable diagnosis for clinical data.

\begin{figure*}[t!]
  \centering
\includegraphics[width=0.9\linewidth]{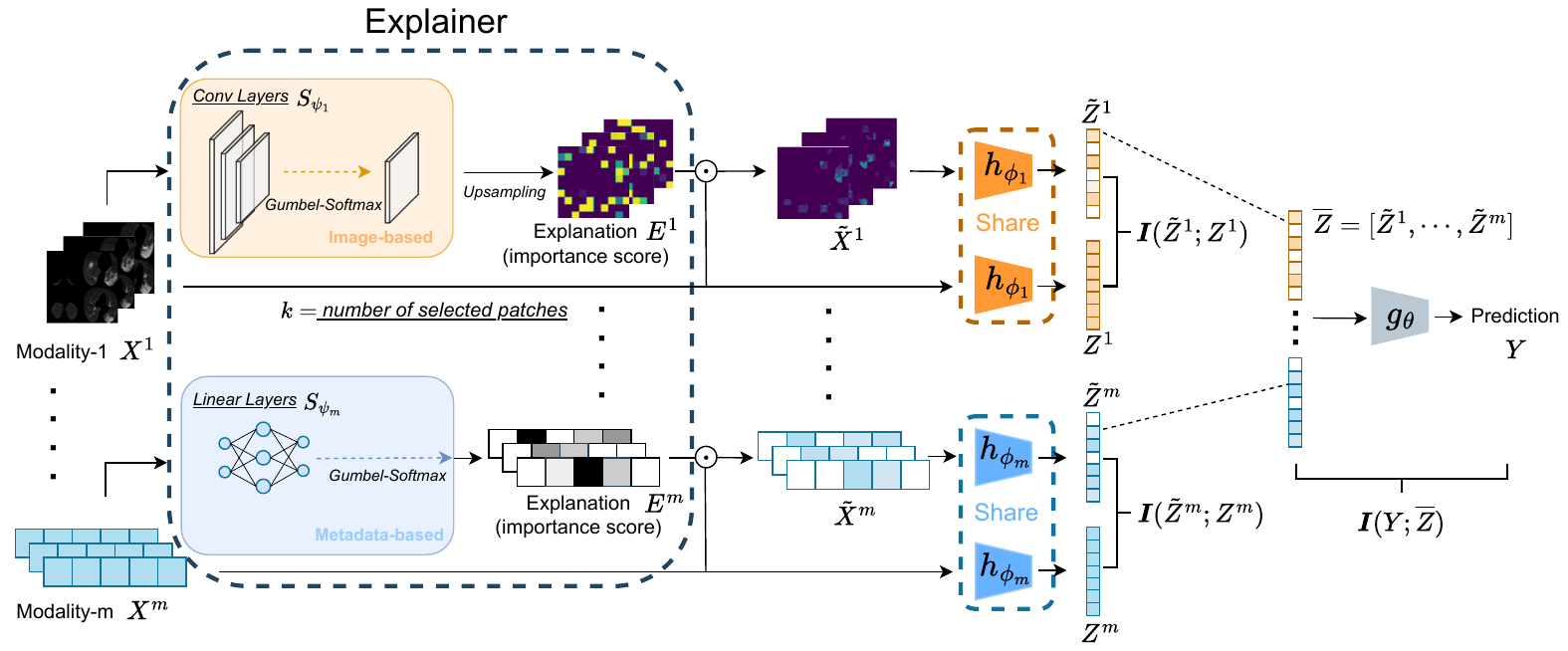}
   \caption{Overview of the proposed SMILE framework. We design a modality-specific explainer 
   to identify the top-$k$ most informative features within each modality that influence diagnostic decisions. All selected features $\{E^m\}_{m=1}^M$ serve as the final explanations and are directly used for prediction, ensuring consistency between the explanations and the model decisions. 
   }
   \label{fig: framework}
\end{figure*}

\section{Methodology}
\label{sec: method}
\subsection{Objective of SMILE}
Given \( N \) i.i.d. multi-modal observations from \( M \) modalities with their corresponding labels \(\{ \{x_i^m\}_{m=1}^M,  y_i\}_{i=1}^N \), let \( x_i^m \) denote the \( i \)-th sample in the \( m \)-th modality. Here, \( y_i \in \mathbb{R}^K \), where \( K \) represents the number of classes. Each \( x_i^m \) can be a vector, such as genomic data; a three-dimensional (3D) volume, such as chest computed tomography (CT) scans; or a graph structure, such as brain networks constructed from functional magnetic resonance imaging (fMRI) signals. Our goal is to learn a predictive model $f: \{X^m\}_{m=1}^M \rightarrow Y$, which also integrates self-explainability in the sense that $f$ is able to identify the most informative input elements from each modality. We therefore introduce a set of $M$ built-in explainers to perform a deterministic mapping $X^m \mapsto E^m$, where $\{e_i^m\}_{i = 1}^N \in E^m$ has the same size as $X^m$ and acts as an explanation to $X^m$. Each element in $E^m$ is a binary variable, with a value of $1$ indicating an informative feature and $0$ indicating a non-informative feature. Typically, $E^m$ is both sample-dependent and modality-dependent.

In order to learn a modality-specific explanation $E^m$, we employ the IB principle. Specifically, let $\tilde{{x}}_i^m = {e}_i^m \odot {x}_i^m$ denote the set of all selected features in the $m$-th modality, where $\odot$ represents element-wise multiplication, the objective for single modality explanation can be expressed as:
\begin{equation}\label{eq: obj1}
    \min_{E^m} - I(Y; \tilde{X}^m) + \beta I(\tilde{X}^m; {X}^m) + \lambda \|{E}^m\|_0,
\end{equation}
where $\|.\|_0$ is the $\ell_0$ norm that encourages the sparsity of $E^m$, only a few elements can be identified as informative, $\beta$ and $\lambda$ are non-negative regularization coefficients.

A closely related objective to our Eq.~(\ref{eq: obj1}) is the INstance-wise VAriable SElection (INVASE)~\cite{yoon2019invase}, which formulates the learning of $E$ in a single modality as:
\begin{equation}\label{eq: obj_INVASE}
    \min_E D_{\text{KL}}\left( p(Y|X) \| p(Y|\tilde{X}) \right) + \lambda \|{E}\|_0,
\end{equation}
where $D_{\text{KL}}$ refers to the Kullback-Leibler (KL) divergence.

\begin{proposition}
    The IB objective in (\ref{eq: obj1}) encompasses that of INVASE as a special case when $\beta=0$.
\end{proposition}

\begin{proof}
    All proofs are provided in the supplementary material.
\end{proof}

One could argue that minimizing $I(X;\tilde{X})$ plays a similar role to the sparsity constraint $\|E\|_0$. In fact, minimizing $I(X;\tilde{X})$ can be viewed as an implicit form of sparsity regularization; however, the two terms play complementary roles. Specifically, the $\ell_0$ constraint explicitly controls the number of selected features, whereas the IB compression term regulates the amount of information preserved in the selected representation. Since the explainer performs deterministic masking, we have $H(\tilde{X}\mid X)=0$, and therefore $I(X;\tilde{X})=H(\tilde{X})$. Hence, minimizing the IB compression term encourages compact, low-entropy representations rather than merely reducing the number of selected features. Generally, higher-dimensional variables have greater entropy because additional dimensions introduce more uncertainty into their joint distribution.


To illustrate the role of both regularization terms, let us consider a toy example in the 2D space, as shown in Fig.~\ref{fig: entropy}. Among the four features $\{x_1,\cdots,x_4\}$, both pairs \( (x_1, x_2) \) and \( (x_3, x_4) \) satisfy $\|E\|_0=2$ and effectively distinguish between the two classes. Consequently, both pairs have high and similar mutual information values \( I(Y; \tilde{X}) \), where $\tilde{X}$ represents $(x_1,x_2)$ or $(x_3,x_4)$. Hence, a regularization term $\|E\|_0$ alone is insufficient to select between \( (x_1, x_2) \) and \( (x_3, x_4) \). However, an entropy minimization penalty further drives our objective to favor \( (x_3, x_4) \) over \( (x_1, x_2) \)\footnote{Entropy is minimized if all points converge to a single point.}. Note that a compressed representation is always a desirable property in downstream applications. Our analysis in Section~\ref{sec: general} will further justify the rationale behind \(I(X; \tilde{X}) \).

Therefore, we argue that a mutual information regularization term \( I(X; \tilde{X}) \) is essential for effective instance-wise feature selection. Our objective in Eq.~(\ref{eq: obj1}) incorporates both terms, achieving effective and efficient feature selection.

Having illustrated the rationale behind each term in Eq.~(\ref{eq: obj1}), it is natural to extend it to the multimodal scenario, in which we aim to learn $\{E^m\}_{m=1}^M$ simultaneously:
\begin{equation}\label{eq: multi_loss}
\begin{split}
    \min_{\{E^m\}_{m=1}^M} & - I(Y; \overline{X}) + \beta\sum_{m = 1}^M I(\tilde{X}^m; {X}^m)  +\lambda \sum_{m=1}^M \|E^m\|_0,\\
     & \text{s.t.,}\:\:  {\overline{X}} = g_\omega({\tilde{X}}^1, \cdots, {\tilde{X}}^M), \:\: \tilde{x}_i^m = {e}_i^m \odot {x}_i^m,
    \end{split}
\end{equation}
where $g_\omega$ refers to a fusion network with parameter $\omega$.
\begin{figure}[t]
  \centering   \includegraphics[width=0.6\linewidth]{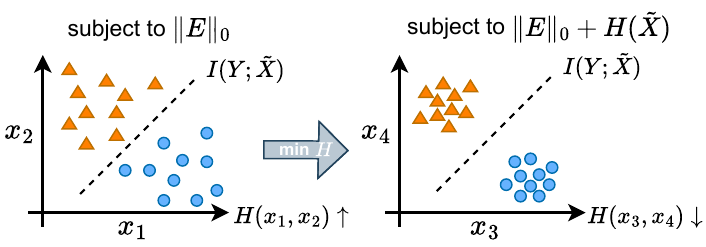}
   \caption{A 2D toy example illustrating how minimizing mutual information $I(X;\tilde{X})$ and entropy $H(\tilde{X})$ aids in feature selection.}
   \label{fig: entropy}
\end{figure}

However, in practical applications, especially in medical diagnosis, where data is often noisy, high-dimensional, and structured, the direct estimation of $I(\tilde{X}^m; {X}^m)$ is infeasible. Inspired by the invariance property of mutual information under reparameterization~\cite{tschannen2019mutual,arjovsky2019invariant}: if $f_1: \mathcal{X} \mapsto \mathcal{X}'$ and $f_2: \mathcal{Y} \mapsto \mathcal{Y}'$ are homeomorphisms (i.e., smooth invertible maps), then it holds that $I(\mathcal{X}; \mathcal{Y}) = I(\mathcal{X}'; \mathcal{Y}')$, we introduce an encoder ${h}_{\phi_m}: X^m \mapsto Z^m$ for each modality, and compute mutual information terms in the transformed, low-dimensional latent space. 
The lossless transformation assumption is also known as the sufficient encoder assumption in contrastive learning~\cite{tian2020makes}. In addition, we implement the $\ell_0$ sparsity regularizer in Eq. (4) as an explicit top-k cardinality constraint. Thus, Eq. (5) should be understood as a constrained implementation of Eq. (4), where the sparsity penalty controlled by $\lambda$ is replaced by the budget $k$:
\begin{equation}\label{eq: multi_loss2}
    \begin{split}
  &  \min_{\{E^m\}_{m=1}^M}  - I(Y; \overline{Z}) + \beta \sum_{m = 1}^M I(h_{\phi_m}(\tilde{X}^m); h_{\phi_m}(X^m)),\\ 
   &  \text{s.t.,}\:\:  {\overline{Z}} = g_\omega( h_{\phi_1}(\tilde{X}^1) , \cdots, h_{\phi_M}(\tilde{X}^M) ),
     \|e_i^m\|_0 \leq k,
    \end{split}
\end{equation}
where $h_{\phi_m}$ denotes a modality-specific encoder with parameters $\phi_m$.
The latent variables $Z^m = h_{\phi_m}(X^m)$ and $\tilde{Z}^m = h_{\phi_m}(\tilde{X}^m)$ facilitate a tractable estimation of mutual information terms. In principle, this strategy can be applied to any modality or data type, provided there is an appropriate encoder to transform complex data into a vector representation. $\|e_i^m\|_0 \leq k$ refers to as the $k$-sparse norm~\cite{makhzani2013k}.

In SMILE, we simply take $g_\omega$ with a concatenation operator, that is $\overline{Z} = [h_{\phi_1}(\tilde{X}^1), \cdots, h_{\phi_M}(\tilde{X}^M)]$.  Hence, our explainable predictive model $f$ consists of three modules, namely the explainer for generating explanations, the encoder for transforming data to latent spaces, and the classifier for predictive tasks, as shown in Fig.~\ref{fig: framework}.

\subsection{Modality-specific, Built-in Explainer}\label{sec: design}

Here, we introduce the design of our modality-specific built-in explainers. 
Our explainers operate directly on the input space of each modality and identify the most informative input elements contributing to the diagnostic decision. Specifically, the definition of an input element is modality-dependent: for image data, the elements can correspond to pixels or image patches; for graph-structured data, they correspond to nodes or edges that form an informative subgraph; and for 3D volumetric data, they correspond to voxels or subvolumes. Therefore, we assume that each modality input consists of $d$ elements, where $d$ is modality-dependent.

For an instance $x_i^m \in \mathbb{R}^d$, given a fixed size $k$ ($0 \le k\le d$), we can define ${\wp}_{k} = \{s_i^m\subset 2^d, |s_i^m| = k\}$, which represents the collection of all possible subsets of elements of size $k$ drawn from a set of $d$ elements. Our aim is to find a small subset $s_i^m \in {\wp}_{k}$ with $k\ll d$ that contains the most informative elements for predicting the model output. We thereby introduce a function $S$, mapping each element in $x_i^m$ to an importance score that indicates its likelihood of being part of the subset $s_i^m$. This importance score is learned through a neural network (i.e., selection network $S_{\psi_m}$ in Fig.~\ref{fig: framework}) parameterized by $\psi_m$, and the subset selection is based on these scores. The architecture of the selection network is modality-dependent and can be adapted to the characteristics of different data types, such as images, graphs, and structured clinical features (see details in Section \ref{sec: ex}).

However, a direct estimation requires summing over $\tbinom{d}{k}$ combinations of feature subsets, which is intractable. Therefore, we employ the Gumbel-Softmax~\cite{jang2016categorical,bang2021explaining} to overcome this problem in a differentiable manner. Suppose we aim to approximate a categorical random variable represented as a one-hot vector in $\mathbb{R}^d$ with category probabilities $ p_1,p_2,\dots, p_d$, we start by adding a random perturbation to the log probability of each category $\log{p_i}$:
\begin{equation}\label{eq: perturbation}
\begin{split}
    G_i = &-\log (-\log u_i) \quad\text{where}\:\: u_i \sim\text{Uniform(0,1)},\\
    &C_i = \frac{\text{exp}\{(G_i + \log{p_i})/\tau\}}{\sum_{i = 1}^d\text{exp}\{(G_i + \log{p_i})/\tau\}},
\end{split}
\end{equation}
where $\tau$ is a tuning parameter for the temperature of the Gumbel-Softmax distribution. Then, we can define a Concrete random vector $C = [C_1,\cdots, C_d]$, which serves as a continuous, differentiable approximation of a categorical random variable represented as a one-hot vector in $\mathbb{R}^d$.

We further define a continuous-relaxed random variable $C^* = [C_1^*, \cdots, C_d^*]$ as the element-wise maximum of the independently sampled Concrete vectors $C^{(j)}$ where $j = 1,\cdots,k$:
\begin{equation}\label{eq: k-hot}
    C_i^* = \max_{j} C_i^{(j)}\:\: \text{i.i.d} \:\: \text{for}\:\: j = 1,...,k.
\end{equation}
We denote this sampling result $C^*$ as the generated explanation $e_i^m \in \mathbb{R}^d$, where top-$k$ elements are most informative for the prediction of the model.

\subsection{Mutual Information Estimation}\label{sec: MI}
We now describe the estimation of mutual information. For the estimation of $I(Y; \overline{Z})$, we can use a standard cross entropy loss~\cite{alemi2016deep}. This is because, letting $Q(Y|\overline{Z})$ be a variational approximation to $P(Y|\overline{Z})$, we have:
\begin{equation}\label{eq: cross_entropy}
    I(Y; \overline{Z}) \ge H(Y) + \mathbb{E}_{P(Y,\overline{Z})}[\log Q(Y|\overline{Z})],
\end{equation}
where $H(Y)$ is independent of the optimization and thereby can be ignored. Maximizing $\mathbb{E}_{P(Y,\overline{Z})}[\log Q(Y|\overline{Z})]$ is equivalent to minimizing the usual cross-entropy loss~\cite{kolchinsky2019nonlinear}.

For the term $I(\tilde{Z}^m; {Z}^m)$, which can be decomposed as:
\begin{equation}\label{eq: I2H}
    I(\tilde{Z}^m; {Z}^m) = H(\tilde{Z}^m) + H(Z^m) - H(\tilde{Z}^m, {Z}^m).
\end{equation}

We employ the matrix-based R\'enyi's $\alpha$-order entropy functional~\cite{giraldo2014measures,yu2019multivariate} to quantify each entropy term in (\ref{eq: I2H}) in a non-parametric manner. Unlike existing parametric estimators such as MINE~\cite{belghazi2018mutual} or InfoNCE~\cite{oord2018representation}, this approach directly estimates entropy from data using positive-definite matrices and avoids introducing an auxiliary parametric model, which complicates training. 


For the $m$-th modality, we denote the set of latent feature vectors from $N$ instances as $\{z_i^m\}_{i=1}^{N}$, where $z_i^m=h_{\phi_m}(x_i^m)\in \mathcal{Z}^m$. The Gram matrix $K \in \mathbb{R}^{N\times N}$ can be computed as $K_{ij} = \kappa({z_i^m}, {z_j^m})$ by using a positive definite kernel $\kappa$. Here, we use the radial basis function (RBF) kernel $\kappa({z_i^m}, {z_j^m}) = \text{exp}(-\frac{{\lVert{z_i^m - z_j^m}\rVert}^2}{2{\sigma}^2})$. The normalized Gram matrix can be defined as $A = K/\text{tr}(K)$, a matrix-based analogue to  Rényi's $\alpha$-order entropy is subsequently defined as:
\begin{equation}\label{eq: entropy}
\begin{split}
    H_{\alpha}(A) 
    = \frac{1}{1-\alpha}\text{log}_2{\left( \sum_{i = 1}^N \lambda_i(A)^{\alpha}\right)}, \quad \alpha \in (0,1) \cup (1, \infty)
\end{split}
\end{equation}
which estimates $H(Z^m)$, and $\lambda_i(A)$ denotes the $i$-th eigenvalue of $A$. With the explainable representation $\tilde{\boldsymbol{z}}^m$, we can apply the same process and obtain $H_{\alpha}(\tilde{A})$, where $\tilde{A}$ is the normalized Gram matrix derived from $\{\tilde{z}_i^m\}_{i = 1}^N \in \tilde{Z}^m$. Additionally, the joint entropy can be defined similarly using the Hadamard product $\odot$:
\begin{equation}\label{eq: jointE}
    H_{\alpha}(A, \tilde{A}) = H_{\alpha}\left(\frac{A\odot \tilde{A}}{\text{tr}(A\odot \tilde{A})}\right),
\end{equation}
which estimates $H(\tilde{Z}^m, {Z}^m)$. Given Eqs.~(\ref{eq: entropy}) and (\ref{eq: jointE}), we can estimate $I(\tilde{Z}^m; {Z}^m)$ with Eq.~(\ref{eq: I2H}).

\subsection{Generalization Analysis}\label{sec: general}

\begin{figure}[t]
  \centering   \includegraphics[width=0.5\linewidth]{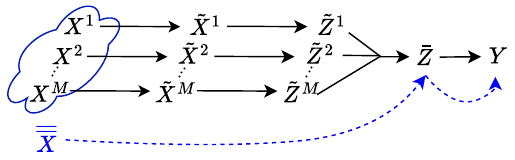}
   \caption{Flow of information in SMILE.}
   \label{fig:flow}
\end{figure}

We finally present a generalization analysis of our SMILE. The information flow within SMILE is illustrated in Fig.~\ref{fig:flow}. Specifically, SMILE consists of $M$ channels, where each explanation $\tilde{X}^m$ and its latent representation $\tilde{Z}^m$ depend solely on $X^m$. SMILE then fuses information from the set $\{\tilde{Z}^m\}_{i=1}^M$ to form $\bar{Z}$, which is then used for prediction. For simplicity, we can also treat $\{X^m\}_{i=1}^M$ as a joint information source, denoted by $\overline{\overline{X}}$, which transmits information through a parameterized channel to $\bar{Z}$ and subsequently to~$Y$. Extending the theoretical result in~\cite{kawaguchi2023does}, we show that the compression term $\sum_{m=1}^M I(\tilde{X}^m;X^m)$ introduced in (\ref{eq: multi_loss}) further reduces the generalization error, as demonstrated in Proposition~\ref{proposition_generalization}.

\begin{proposition}
\label{proposition_generalization}
With probability at least $1-\delta$ over the training data $t= \{ \overline{\overline{x}}_i,  y_i\}_{i=1}^N $ drawn from a data distribution $p(\overline{\overline{x}},y)$, where $ \overline{\overline{x}}_i = \{x_i^m\}_{m=1}^M$, the generalization error
    $\Delta(t) = \mathbb{E}_{p(\overline{\overline{x}},y)} \left[\ell(f^t(\overline{\overline{x}},y)) \right] - \frac{1}{N}\sum_{i=1}^N  \ell(f^t(\overline{\overline{x}}_i,y_i)) $ roughly obeys the following form:
\begin{equation}
    \Delta(t) = \tilde{\mathcal{O}} \left( \sqrt{ \frac{\sum_{m}^M I(\tilde{X}^m;X^m) +1}{N} } \right) \:\: \text{as} \:\: N \rightarrow \infty,
\end{equation}
where $\ell$ is a bounded per-sample loss, $f^t$ is the full model obtained by training over $t$.
\end{proposition}

It is widely believed that as model complexity increases to improve performance, interpretability often declines, leading to a trade-off where more interpretable models may sacrifice some predictive accuracy~\cite{minh2022explainable}. However, our theoretical analysis suggests that the IB-based compression mechanism does not necessarily introduce an intrinsic accuracy penalty, which is consistent with recent perspectives~\cite{rudin2022interpretable,hou2024self}. 
We provide empirical evidence to support this argument in Section~\ref{sec: ex}.

\section{Experiments and Results}\label{sec: ex}

\subsection{Experimental Setup}

\textbf{Datasets.} To demonstrate the effectiveness of SMILE, we evaluate it on five representative datasets that cover commonly used medical multimodal settings: 
\romannumeral1) \textbf{BRCA}~\cite{wang2021mogonet} (breast carcinoma PAM50 subtypes) and \textbf{ROSMAP}~\cite{wang2021mogonet} (Alzheimer’s Disease), where each subject is characterized by three structured-data modalities, i.e., mRNA expression, DNA methylation, and miRNA expression; 
\romannumeral2) \textbf{iCTCF}~\cite{ning2020open} (COVID-19 morbidity), which contains paired HRCT scans and clinical features; following~\cite{fang2024dynamic}, each 3D HRCT volume is preprocessed into a $700\times700$ 2D montage for model input; 
\romannumeral3) \textbf{Glaucoma Grading}~\cite{wu2023gamma}, a dual-image-modality dataset consisting of fundus images (resized to $256\times256$) and OCT images (resized to $512\times512$) for glaucoma severity grading; 
\romannumeral4) \textbf{REST-meta-MDD}~\cite{yan2019reduced}, a large-scale multi-site resting-state consortium for major depressive disorder (MDD), aggregated from 25 cohorts. For sMRI, we use $z$-normalized gray matter volume (GMV) maps of size ($121\times145\times121$). For rs-fMRI, we construct functional connectivity matrices by computing pairwise Pearson correlations between the time series of 116 brain regions defined by the Automated Anatomical Labeling (AAL) atlas~\cite{tzourio2002automated}, followed by Fisher-$z$ transformation. Each subject is therefore represented by a $116\times116$ functional connectivity matrix. After quality control (QC), 1,604 subjects are retained. See Table~\ref{tab: datasets} for characteristics of datasets.


\textbf{Implementation Details.} 
In the explainer module, we employ modality-specific selection networks to accommodate the heterogeneous characteristics of different data modalities. The detailed architectures and implementation settings of these selection networks are provided in the supplementary material. For image data, the explainer operates at the patch level rather than individual pixels and employs convolutional layers to capture local spatial patterns. For structured data, such as clinical variables and genomic features, fully connected networks are used to estimate feature importance. For graph-structured data, the explainer evaluates the importance of graph elements (i.e., edges) using modality-specific networks; in our implementation, the edge importance is estimated by an MLP based on the features of the corresponding node pairs.
In the fusion stage, we use simple concatenation to fuse latent representations from each modality. 


For training the proposed method on \textbf{BRCA} and \textbf{ROSMAP}, we follow the experimental settings in~\cite{han2022multimodal} and set $k=30$. 
For \textbf{iCTCF}, we adopt the settings from~\cite{fang2024dynamic} and use $k=30$ patches of size $70\times70$ for the image modality and $k=20$ for clinical features. 
For \textbf{Glaucoma Grading}, we use the experimental setup in~\cite{fang2021multi} and set $k=100$, using $8\times8$ patches for fundus images and $4\times4$ slices for OCT images. For \textbf{REST-meta-MDD}, we use soft top-$k$ selection with $k=25$ edges (rs-fMRI; $\sim$1\% of admissible pairs) and $k=25$ subvolumes (sMRI; $\sim$5\% of a $7\times9\times7$ grid); in all cases, $k$ is chosen empirically as the smallest value before performance drops. All experiments are performed on a single H100 GPU (80GB, cuDNN 9.3), and each experiment is repeated 10 times to report the mean and standard deviation.

\begin{table*}[t!]
\centering
\large
\caption{Dataset characteristics. Note that we use a prognosis task for \textbf{iCTCF} rather than a diagnostic task to identify infection, to better showcase explanations.}
\resizebox{\linewidth}{!}{%
\renewcommand{\arraystretch}{1.1}
\begin{tabular}{llll}
\toprule
Dataset          & Modality                         & Task                                         & Number of labels and patients                               \\
\midrule
BRCA             & mRNA, DNA methylation, miRNA     & Diagnosis for breast carcinoma PAM50 subtype & Normal: 115 / Basal: 131 / Her2: 46 / LumA: 436 / LumB: 147 \\
ROSMAP           & mRNA, DNA methylation, miRNA     & Diagnosis for Alzheimer’s Disease            & Normal: 169 / AD: 182                                       \\
iCTCF            & HRCT scans, 81 clinical features & Prognosis for COVID-19 morbidity outcome      & Severe symptoms: 202 / Mild symptoms: 549                   \\
Glaucoma Grading & 2D fundus images, 3D OCT scans   & Diagnosis for Glaucoma Grad                  & Non: 50 / Early: 26 / Mid advanced: 24 \\
REST-meta-MDD &rs-fMRI graphs, 3D sMRI graphs& Diagnosis for Major Depressive Disorder& Major Depressive Disorder: 1300 / Healthy Controls: 1128 (830 / 771 after QC)\\
\bottomrule
\end{tabular}
}
\label{tab: datasets}
\end{table*}

\subsection{Quantitative Analysis}
\noindent\textbf{BRCA \& ROSMAP.}
For these multi-omics datasets, we compare against five representative
integration methods covering the main families: classical statistical
integration (DIABLO~\cite{singh2019diablo}), graph-based deep integration (MOGONET~\cite{wang2021mogonet}), uncertainty-aware fusion (TMC~\cite{han2022trusted}, Dynamic~\cite{han2022multimodal}, and DMIB~\cite{fang2024dynamic}, the closest competitor as it is also IB-based but lacks any explanation mechanism. As shown in Table~\ref{tab: BRCA}, SMILE achieves the best ACC, WeightedF1, and MacroF1 on BRCA ($87.3\%$, $87.8\%$, $84.6\%$); the largest margin appears on MacroF1, indicating that the selected features remain discriminative for minority PAM50 subtypes. On the binary ROSMAP task, SMILE leads in ACC ($85.1\%$) and F1 ($85.6\%$), while DMIB retains the best AUC ($91.6\%$ vs.\ $90.3\%$). SMILE thus matches or exceeds the strongest IB-based competitor while additionally providing instance-wise explanations.

\begin{table*}[thp]
\centering
\tiny
\caption{Performance of various multimodal methods on \textbf{BRCA} and \textbf{ROSMAP} datasets. The best results are in bold, and the second-best results are underlined.}
\resizebox{0.8\linewidth}{!}{%
\renewcommand{\arraystretch}{1.1}
\begin{tabular}{c|ccc|ccc}
\toprule
        & \multicolumn{3}{c}{BRCA}         & \multicolumn{3}{|c}{ROSMAP}     \\
\midrule
Method  & ACC      & WeightedF1 & MacroF1  & ACC      & F1       & AUC      \\
\midrule
DIABLO~\cite{singh2019diablo}  & 64.2$\pm$0.9 & 53.4$\pm$1.7   & 36.9$\pm$1.7 & 74.2$\pm$2.4 & 75.5$\pm$2.5 & 83.0$\pm$2.5 \\
MOGONET~\cite{wang2021mogonet} & 82.9$\pm$1.8 & 82.5$\pm$1.7   & 77.4$\pm$1.7 & 81.5$\pm$2.3 & 82.1$\pm$1.2 & 87.4$\pm$1.2 \\
TMC~\cite{han2022trusted}     & 84.2$\pm$0.5 & 84.4$\pm$0.9   & 80.6$\pm$0.9 & 82.5$\pm$0.9 & 82.3$\pm$0.6 & 88.5$\pm$0.6 \\
Dynamic~\cite{han2022multimodal} & \underline{87.1$\pm$0.5} & \underline{87.4$\pm$0.6}   & \underline{83.5$\pm$0.5} & 81.7$\pm$1.5 & 82.3$\pm$1.5 & 90.0$\pm$1.2 \\
DMIB~\cite{fang2024dynamic}   & 86.0$\pm$0.7 & 86.0$\pm$0.8   & 81.6$\pm$0.9 & \underline{84.9$\pm$1.8} & \underline{85.3$\pm$1.7} & \textbf{91.6$\pm$0.7} \\
SMILE \textit{(Ours)}    &    \textbf{87.3$\pm$0.2}      & \textbf{87.8$\pm$0.3}           &   \textbf{84.6$\pm$0.3}       &   \textbf{85.1$\pm$1.1}       &  \textbf{85.6$\pm$1.3}        &   \underline{90.3$\pm$0.4}      \\
\bottomrule
\end{tabular}
}
\label{tab: BRCA}
\end{table*}

\begin{table}[thp]
\centering
\caption{Performance of various multimodal methods on \textbf{iCTCF} dataset. SMILE results are reported as the mean and standard deviation across 10 runs of 5-fold cross-validation.}
\resizebox{0.7\linewidth}{!}{%
\renewcommand{\arraystretch}{1.1}
\begin{tabular}{c|cccc}
\toprule
               & ACC  & AUC  & JW$_{0.5}$   & JW$_{0.6}$   \\
\midrule
De-COVID19-Net~\cite{meng2020deep} & 72.4 & 77.3 & 69.8 & 69.2 \\
HoFN+SCResNet~\cite{zhou2021cohesive}  & 68.4 & 80.2 & 70.1 & 71.6 \\
HUST-19~\cite{ning2020open}&\underline{83.3}&\underline{92.1}& --& --\\
DMIB~\cite{fang2024dynamic}           & 73.2 & 82.3 & 72.1 & 71.7 \\
SMILE \textit{(ours)}   & \textbf{92.4$\pm$0.5} & \textbf{95.6$\pm$0.6} & \textbf{84.3$\pm$0.5} & \textbf{82.2$\pm$0.3}\\
\bottomrule
\end{tabular}
}
\label{tab: iCTCF}
\end{table}

\begin{table}[thp]
\centering
\caption{Performance comparison of various multimodal methods on the \textbf{REST-meta-MDD} dataset (\%). Best results are highlighted in bold. Our proposed SMILE achieves compelling performance in terms of all four metrics.}
\resizebox{0.7\linewidth}{!}{%
\renewcommand{\arraystretch}{1.1}
\begin{tabular}{c|cccc}
\toprule
               & ACC  & F1-score  & AUC   & MCC   \\
\midrule
DER~\cite{amini2020deep} &56.3	&41.3	&66.8	&17.4 \\
MoNIG~\cite{ma2021trustworthy} &\underline{63.6}	&64.8	&67.9&\underline{27.0} \\
MIB~\cite{mai2022multimodal} &55.9	&70.1	&77.7	&21.3 \\
MEIB~\cite{zhang2022multi}	&52.2	&\underline{68.5}	&\textbf{82.2}	&6.7 \\
SMILE \textit{(ours)}   & \textbf{67.4$\pm$2.3} & \textbf{69.7$\pm$2.6}   &\underline{73.5$\pm$3.6}& \textbf{35.0$\pm$4.6} \\
\bottomrule
\end{tabular}
}
\label{tab: REST-meta-MDD}             
\end{table}

\noindent\textbf{iCTCF.}
As this dataset pairs imaging with clinical features and has dedicated models built on it, we compare against three COVID-19-specific prognosis models (De-COVID19 Net~\cite{meng2020deep}, HoFN+SCResNet~\cite{zhou2021cohesive}, and HUST-19~\cite{ning2020open}, released with the dataset) plus the general-purpose DMIB~\cite{fang2024dynamic}, and follow~\cite{fang2024dynamic} in adopting ACC, AUC and the weighted Youden indices $J_{W0.5}$, $J_{W0.6}$, which reflect the asymmetric clinical cost of missing severe cases. Table~\ref{tab: iCTCF} shows that SMILE raises ACC from $83.3\%$ to $92.4\%$ ($9.1$ percentage points over HUST-19) and attains the highest AUC ($95.6\%$); among methods reporting the Youden indices, it also obtains the best $J_{W0.5}$ ($84.3\%$) and $J_{W0.6}$ ($82.2\%$), confirming that the gain is not driven by the majority class.

\begin{table}[t!]
\centering
\caption{Performance of different strategies on \textbf{Glaucoma Grading} dataset. To compare directly with pre-trained models used for the classification task~\cite{wu2023gamma}, we use them as encoders.}
\resizebox{0.7\linewidth}{!}{%
\renewcommand{\arraystretch}{1.1}
\begin{tabular}{c|cc|cc}
\toprule
      & DuelRes      & DuelRes (SMILE) & Res-DEN      & Res-DEN (SMILE) \\
      \midrule
Kappa & 65.4$\pm$1.2 & \textbf{69.2$\pm$0.8}    & 70.1$\pm$0.5 & \textbf{72.4$\pm$0.5}   \\
\bottomrule
\end{tabular}
}
\label{tab: Glaucoma}
\end{table}

\noindent\textbf{Glaucoma Grading.} To compare with baseline models~\cite{wu2023gamma}, we replace encoders with pre-trained models (excluding the classifier layer) and keep Cohen’s Kappa as the evaluation metric. DuelRes with SMILE achieves a notable improvement with a Kappa score of 69.2\% compared to 65.4\% and presents better stability. Similarly, Res-DEN with SMILE reaches a Kappa score of 72.4\% compared to 70.1\% (see Table~\ref{tab: Glaucoma}). This demonstrates that involving the pre-trained models in our framework consistently improves performance.

\noindent\textbf{REST-meta-MDD.} Table~\ref{tab: REST-meta-MDD} compares SMILE against
four multimodal baselines on the REST-meta-MDD dataset. SMILE achieves the highest accuracy (67.4\%), exceeding the strongest baseline in accuracy, MoNIG~\cite{ma2021trustworthy}, by 3.8 percentage points, and the largest margin appears on MCC (35.0\% vs.\ 27.0\%), the metric least affected by majority-class bias. Its F1-score (69.7\%) is on par with the best baseline
(MIB, 70.1\%). Two baselines attain higher AUC, yet this does not translate into usable decisions: MEIB~\cite{zhang2022multi} reaches the highest AUC (82.2\%) but collapses to 52.2\% accuracy with near-chance MCC (6.7\%), and MIB~\cite{mai2022multimodal} likewise pairs a high AUC (77.7\%) with 55.9\% accuracy and an MCC of 21.3\%. Similarly, DER~\cite{amini2020deep} shows a wide gap between accuracy and F1-score (56.3\% vs.\ 41.3\%). These results suggest that SMILE provides a more balanced operating point under the default decision threshold.


\subsection{Qualitative Analysis}
The most important innovation of SMILE is its ability to generate modality-specific explanations. We present qualitative examples to show its effectiveness.

For genomic data, we visualize the top-5 biomarkers identified by SMILE across three modalities of \textbf{BRCA}, which is shown at the top of Fig.~\ref{fig: gene_explanation}. Meanwhile, we reference some representative studies to illustrate the relevance of these biomarkers to breast cancer progression. Specifically, for mRNA expression in Fig.~\ref{fig: RNA}, in HER2‐positive breast cancers, one of the most aggressive subtypes of breast cancer, \textit{KCTD10} can induce RhoB degradation and activation of Rac1~\cite{angrisani2021emerging}. \textit{TEX10} is illustrated in~\cite{wang2020rnai}, which finds that its interaction with the NF-$\kappa$B pathway promotes tumor progression by enhancing the expression of key genes involved in inflammation and survival. \cite{lee2024phgdh} claims that the degradation of \textit{PHGDH} by ring finger protein 5 inhibited breast cancer cell growth. \textit{TP53INP1} is an independent risk factor for overall survival in breast cancer patients~\cite{nishimoto2019prognostic}. For DNA methylation expression in Fig.~\ref{fig: DNA}, \textit{ZCCHC11}~\cite{zhang2022terminal}, also known as TUT4, has been verified to be associated with breast cancer. \textit{TEX19}~\cite{liu2024tex19}, \textit{SNORD82}~\cite{karkkainen2022expression} and \textit{FZD7}~\cite{yang2011fzd7} are relevant to breast cancer and its diagnosis. Similar to miRNA expression in Fig.~\ref{fig: miRNA}, \textit{miR-520b}~\cite{lu2017mir}, \textit{miR-874}~\cite{aersilan2022microrna} and \textit{miR-665}~\cite{zhao2019mir}, all of these identified gene sequences are associated with breast cancer.

In the \textbf{ROSMAP} dataset, the top-5 biomarkers identified by SMILE across three modalities are shown at the bottom of Fig.~\ref{fig: gene_explanation}. These biomarkers are supported by leading medical journals linking them to Alzheimer’s Disease (AD). For mRNA in Fig.~\ref{fig: RNA}, \textit{QDPR} has been observed to be defective in the AD brain because it catalyzes the conversion of BH4 away from these neurotransmitters~\cite{xu2019regional, mathys2019single}. 
\cite{canchi2019integrating,perez2021rtp801}
indicate that \textit{DDIT4} is differentially expressed in the prefrontal cortex of AD patients. \textit{SLC5A11} is included in the top 5 cluster-enriched transcripts per astrocytes and oligodendrocytes which change in AD~\cite{sadick2022astrocytes}, while \textit{MPV17L2} is also recognized as a high-confidence AD-associated gene~\cite{zhang2024g}. For DNA methylation in Fig.~\ref{fig: DNA}, the dataset contains Cytosine-phosphate-Guanine (GpG) sites (e.g., \textit{cg22472290}). We are able to find the relevance of their associated genes to the task, such as \textit{cg22472290} in \textit{2NF577}~\cite{lardenoije2019alzheimer}, \textit{cg17253459} in \textit{OLFML3}~\cite{drummond2022amyloid}. 
The identified miRNAs in Fig.~\ref{fig: miRNA}, including \textit{miR-199a-5p}~\cite{song2020mir}, \textit{miR-346}~\cite{long2019novel}, \textit{miR-484}~\cite{jia2022mir}, \textit{miR-129-3p}~\cite{patrick2017dissecting} and \textit{miR-133a}~\cite{chum2022cerebrovascular}, have all been associated with breast cancer.


\begin{figure}[t!]
  \centering
  \subfloat[\label{fig: RNA}]{
  \begin{minipage}[b]{0.32\linewidth}
    \centering
    \includegraphics[width=\linewidth]{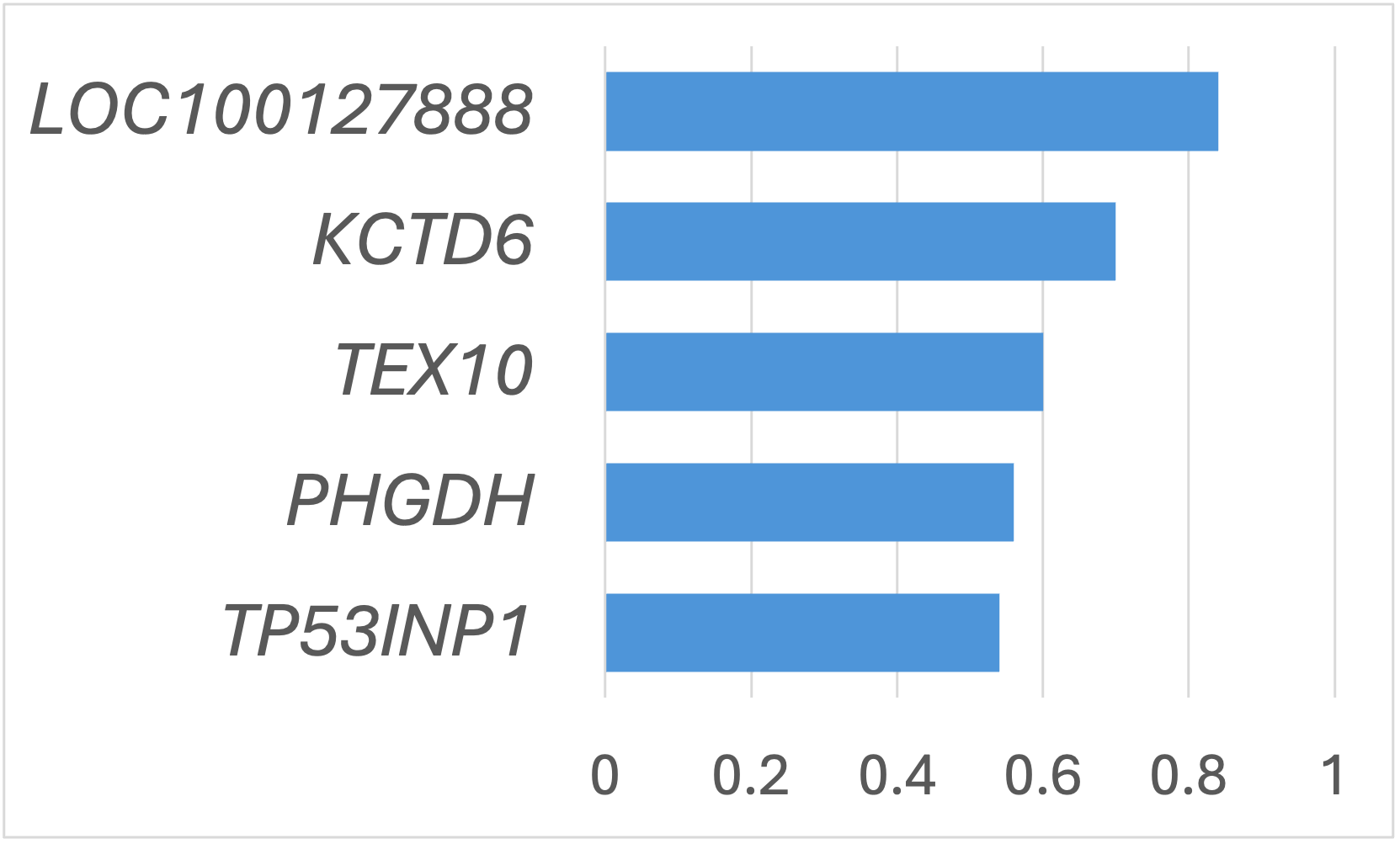}\\[2mm]
    \includegraphics[width=\linewidth]{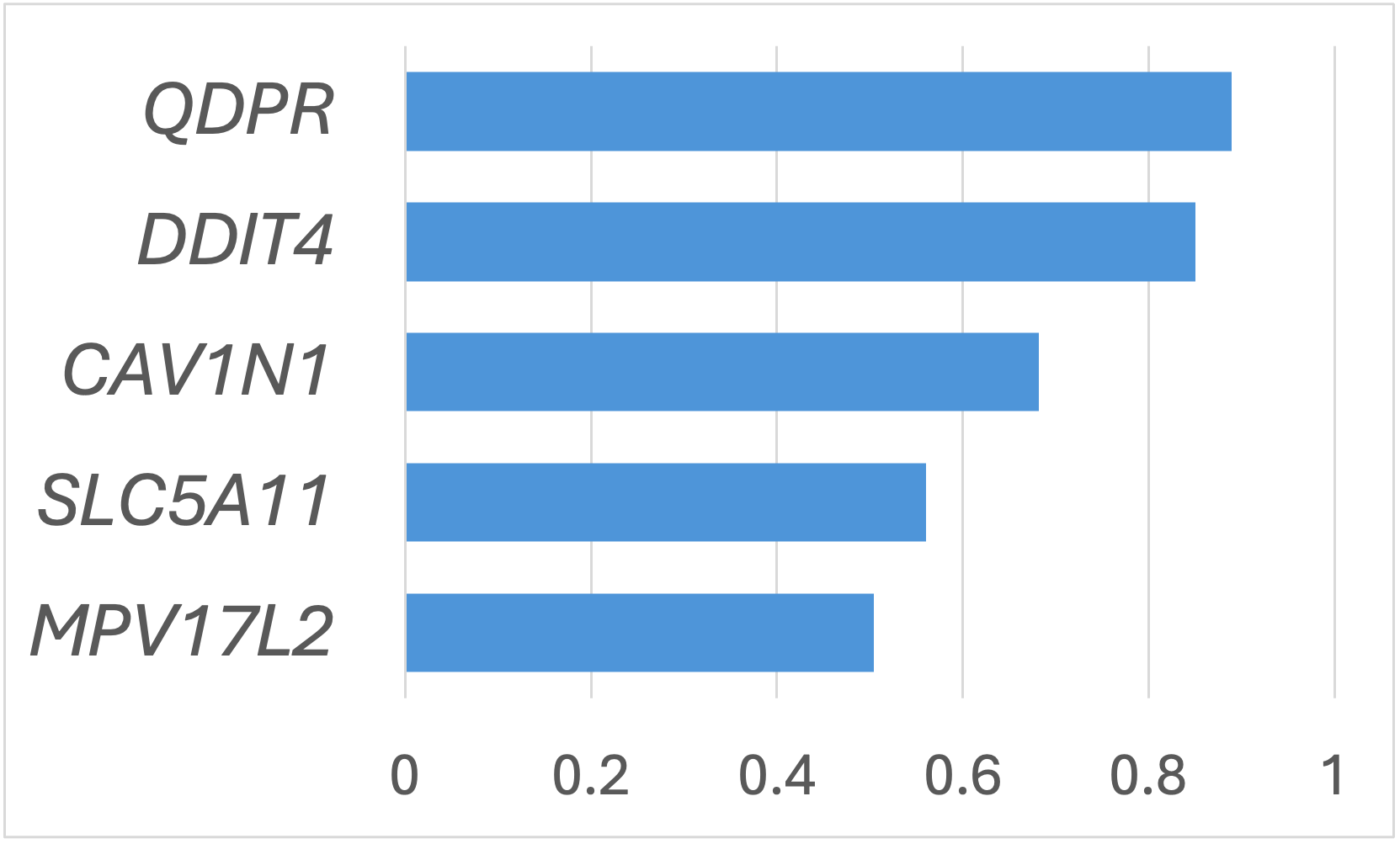}
  \end{minipage}}
  \subfloat[\label{fig: DNA}]{
   \begin{minipage}[b]{0.32\linewidth}
    \centering
    \includegraphics[width=\linewidth]{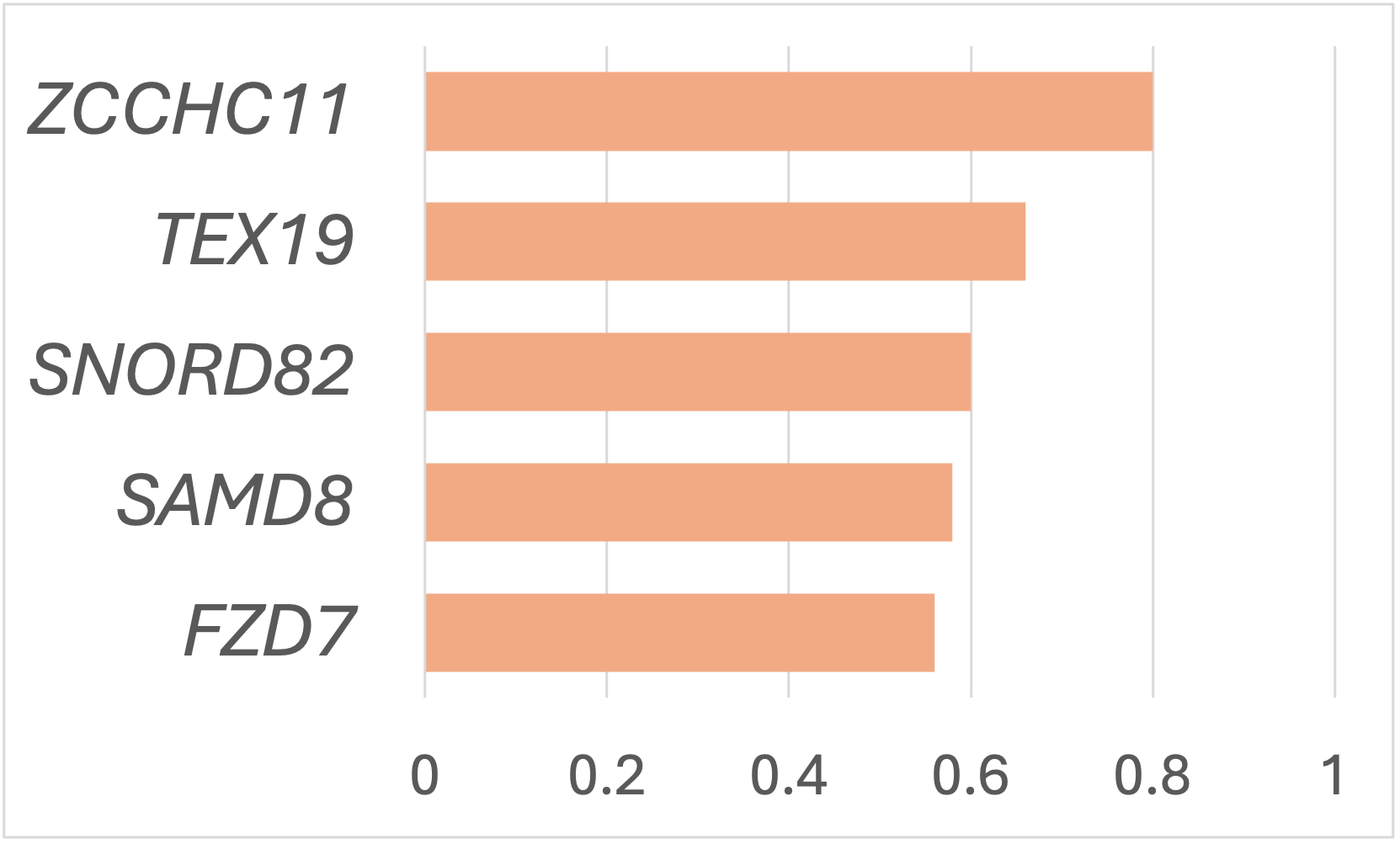}\\[2mm]
    \includegraphics[width=\linewidth]{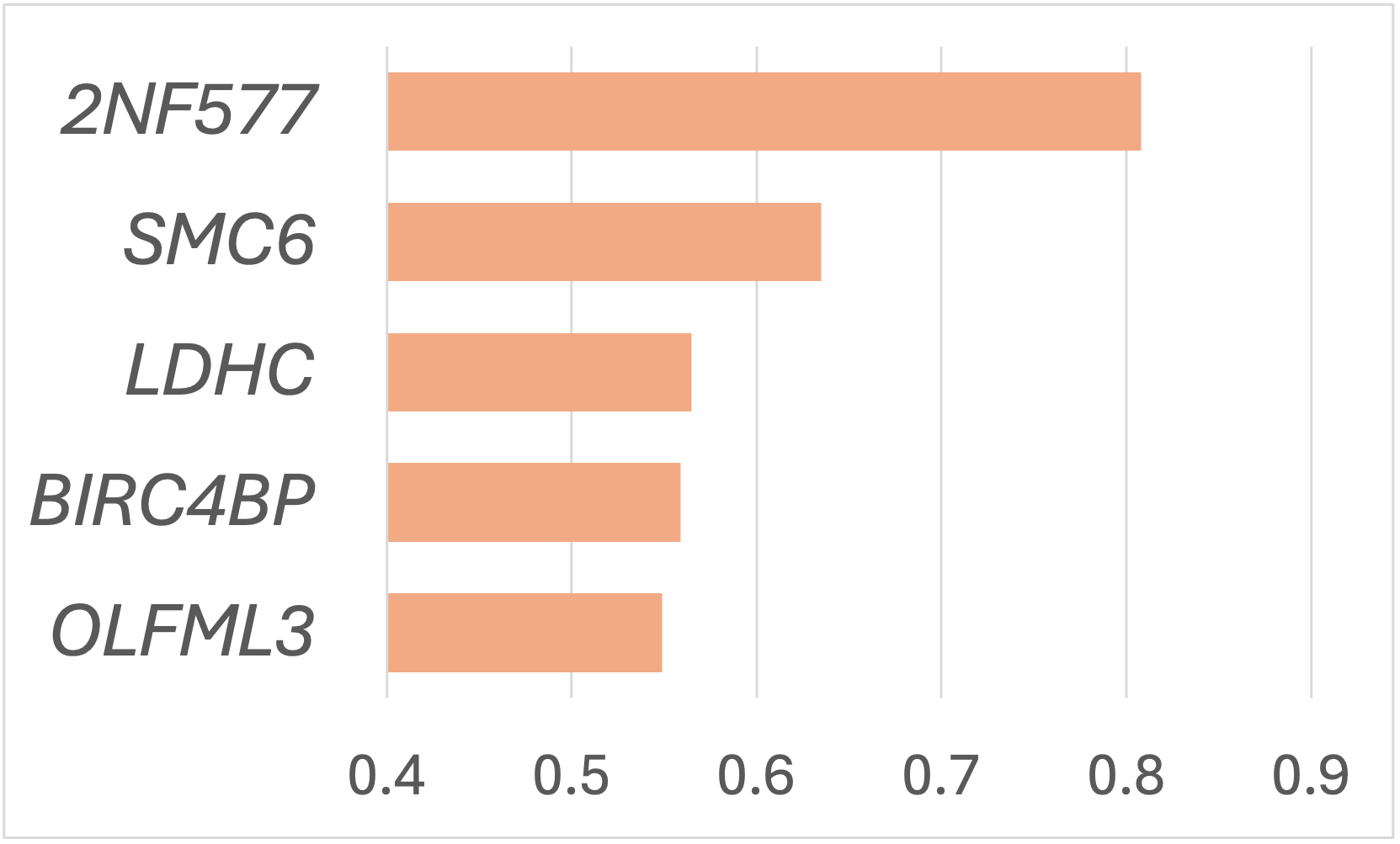}
  \end{minipage}}
  \subfloat[\label{fig: miRNA}]{
 \begin{minipage}[b]{0.32\linewidth}
    \centering
    \includegraphics[width=\linewidth]{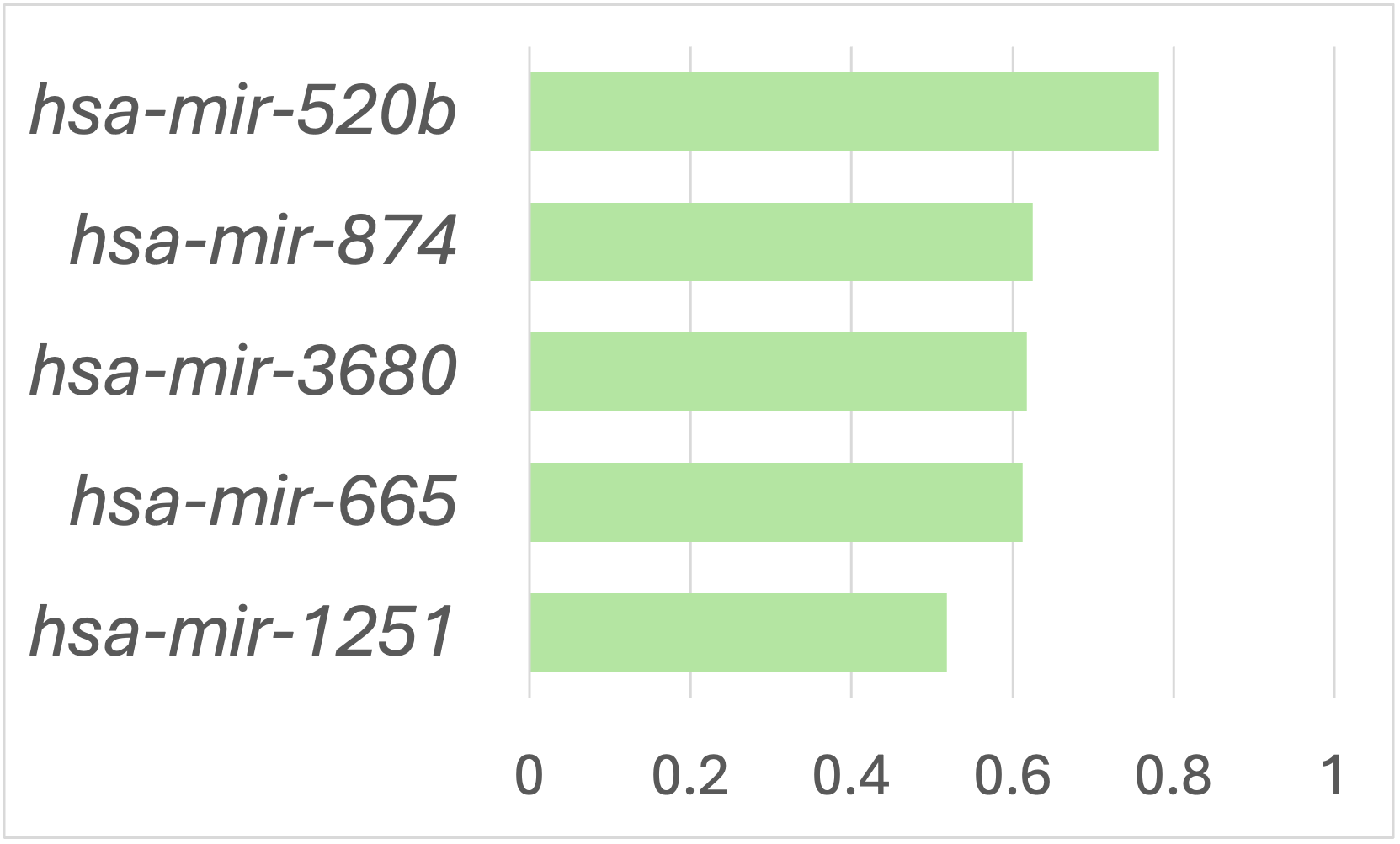}\\[2mm]
    \includegraphics[width=\linewidth]{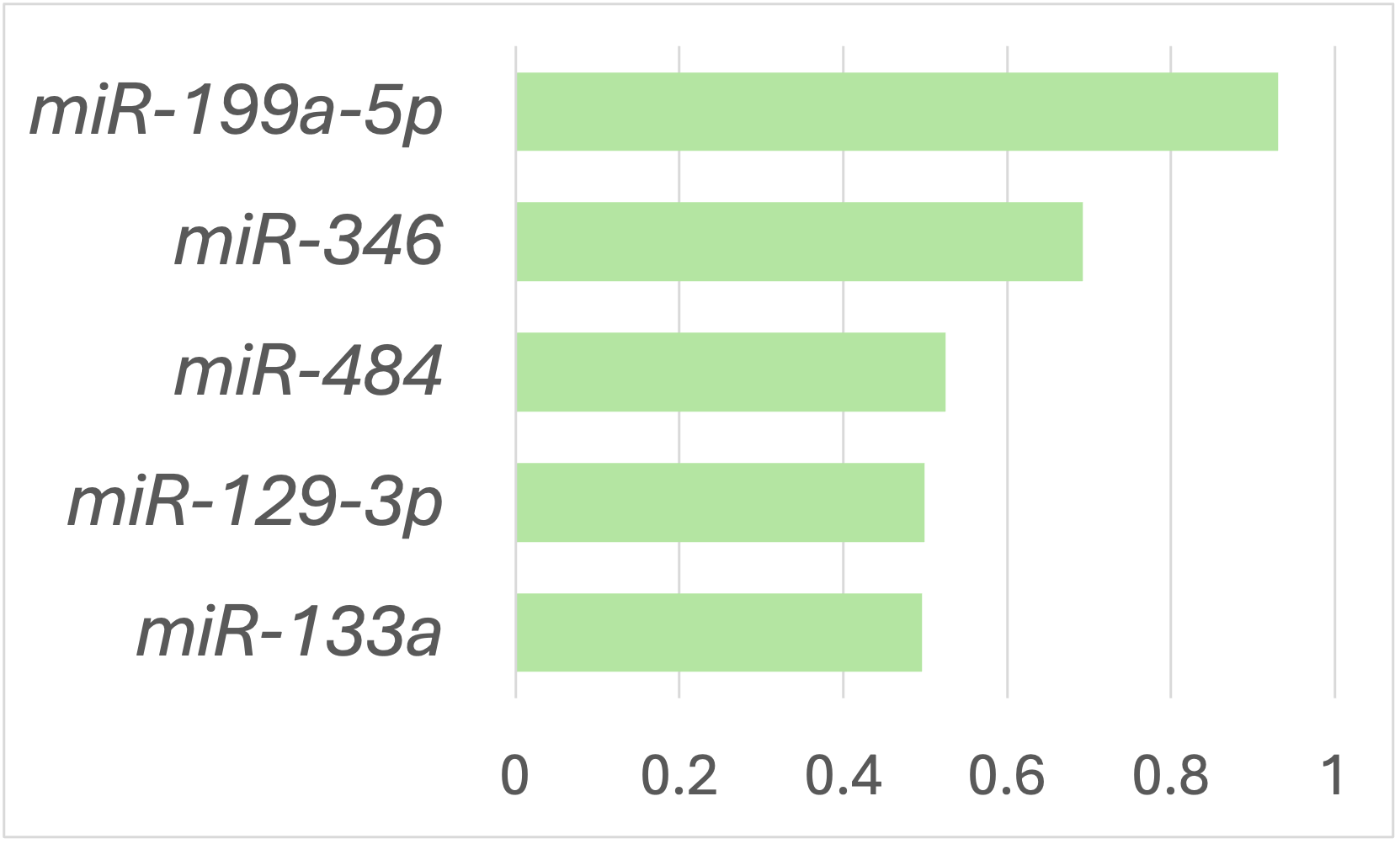}
  \end{minipage}}
  \caption{The top-5 biomarkers among the $k=30$ features selected by SMILE are shown for two datasets (top row: \textbf{BRCA}; bottom row: \textbf{ROSMAP}) across three modalities: (a) mRNA expression, (b) DNA methylation, and (c) miRNA expression.}
  \label{fig: gene_explanation}
\end{figure}

Fig.~\ref{fig: iCTCF_visual} presents qualitative examples of \textbf{iCTCF}, which contains two modalities. We find that the generated visual explanations (see Fig.~\ref{fig: iCTCF_image}) successfully localize COVID-19-related lung features, such as Consolidation, Ground-Glass Opacities (GGOs), and Crazy-paving patterns, for 2D montage images. For selected clinical information in Fig.~\ref{fig: iCTCFvector}, ALG, LYP, CRP, GGT, LDH, NEP, and AST are task-relevant references~\cite{ning2020open}.
\begin{figure*}[thp]
\centering
  \subfloat[\label{fig: iCTCF_image}]{
  \centering
    \includegraphics[width=\linewidth]{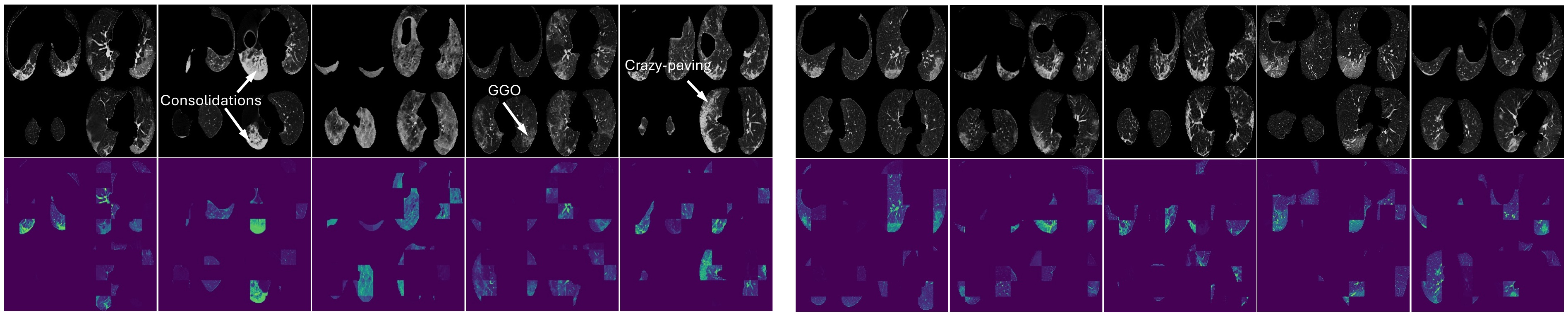}}
  \hfill
  \subfloat[\label{fig: iCTCFvector}]{
  \centering
    \includegraphics[width=\linewidth]{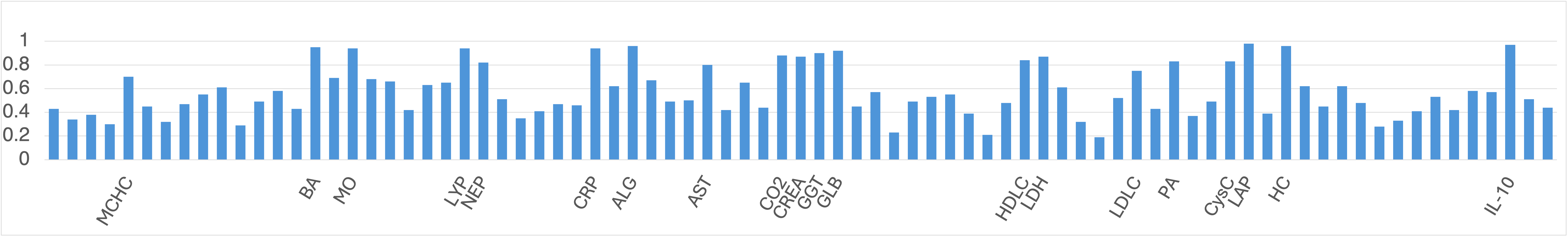}}
   \caption{\textbf{iCTCF} visualizations on two modalities. (a) Qualitative examples of CT scan in the 2D montage: each explanation contains $k = 30$ patches. The left rows present severe patients' samples; the right rows present mild patients' samples. (b) Qualitative examples of clinical information: normalized results from 150 test samples, $k = 20$ features selected.}
   \label{fig: iCTCF_visual}
\end{figure*}

Fig.~\ref{fig: Glaucoma} presents visual explanations for two modalities in \textbf{Glaucoma Grading}, demonstrating that SMILE effectively identifies key features in both multi-channel and single-channel images.
\begin{figure*}[t]
\centering
  \subfloat[\label{fig: Glaucoma_fundus}]{
  \centering
    \includegraphics[width=0.47\linewidth]{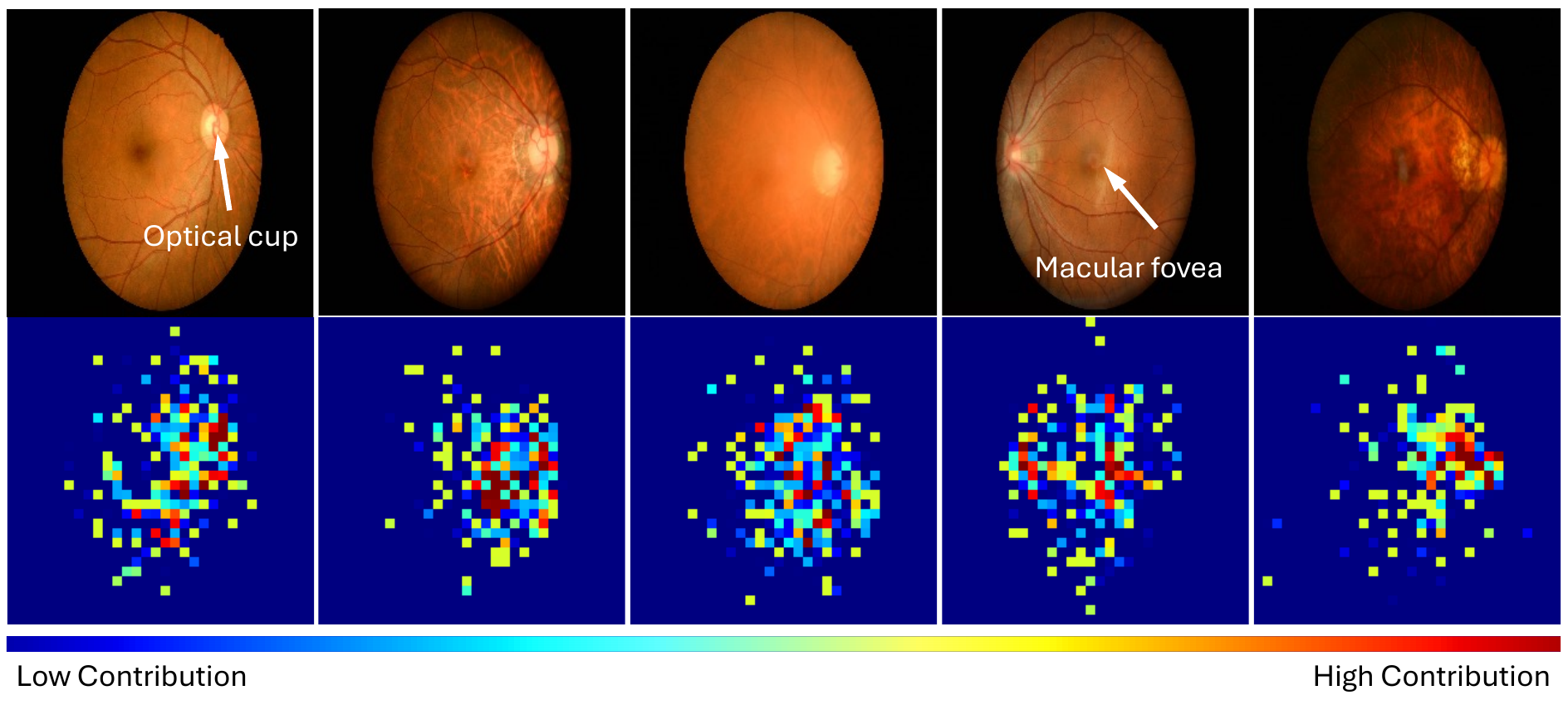}}
  \hspace{0.3mm}
  \subfloat[\label{fig: Glaucoma_CT}]{
  \centering
\includegraphics[width=0.47\linewidth]{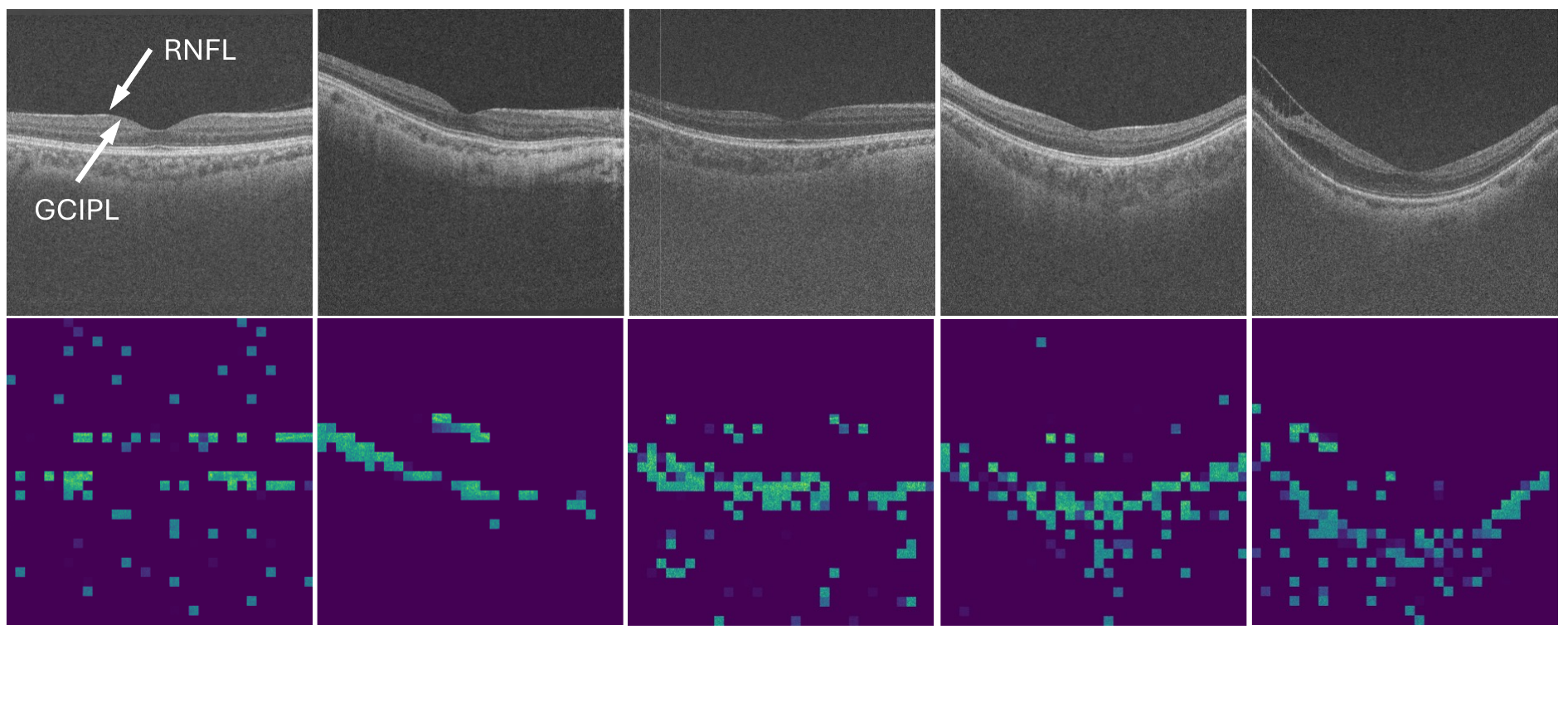}}
   \caption{\textbf{Glaucoma Grading} visualizations on two modalities. (a) Qualitative examples of color fundus ($k = 100$); (b) Qualitative examples of OCT slices ($k = 100$)}
   \label{fig: Glaucoma}
\end{figure*}


Fig.~\ref{fig:REST-meta} demonstrates the extracted explanations from both fMRI and sMRI modalities on the \textbf{REST-meta-MDD} dataset. To examine what SMILE attends to at the population level, we aggregated per-subject selections into frequencies within each group. The structural selector returned sparse, anatomically coherent subvolumes rather than scattered voxels, and the functional selector a compact rather than diffuse edge set. In six of ten runs, structural selection converged on a posterior occipital--parietal territory, indicating that the pattern reflects the learned model rather than a single initialization. This territory coincides with regions of reduced cortical thickness identified in a vertex-wise meta-analysis of 64 cohorts from the ENIGMA MDD and DIRECT consortia---the latter supplied our data---including the lateral occipital and parietal cortices~\cite{yan2026vertex}, and with morphometric findings on REST-meta-MDD itself~\cite{wang2023brain}. The selected edges center on occipito-temporal pathways, complementing the default-mode alterations reported previously in this cohort~\cite{yan2019reduced}. One caveat: the HC and MDD maps overlap substantially, so they evidence anatomically plausible and reproducible attention, not group-discriminative biomarkers. The co-localization of sMRI subvolumes with fMRI edge endpoints nonetheless indicates that the model draws on convergent structural and functional evidence.

\begin{figure*}[t]
\centering
  \subfloat[\label{fig: HC}]{
  \centering
    \includegraphics[width=0.47\linewidth]{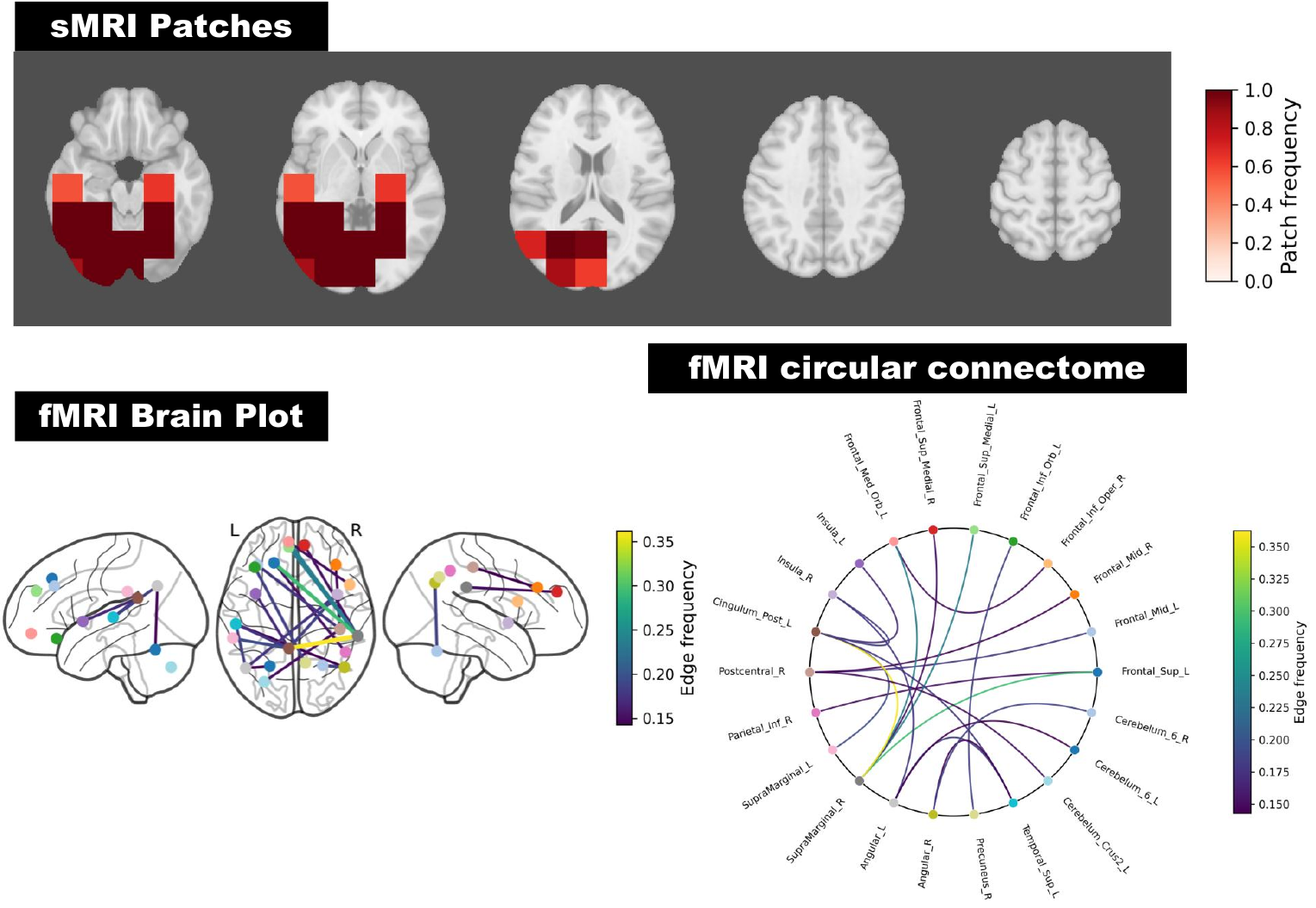}}
  \hspace{0.4mm}
  \subfloat[\label{fig: MDD}]{
  \centering
\includegraphics[width=0.47\linewidth]{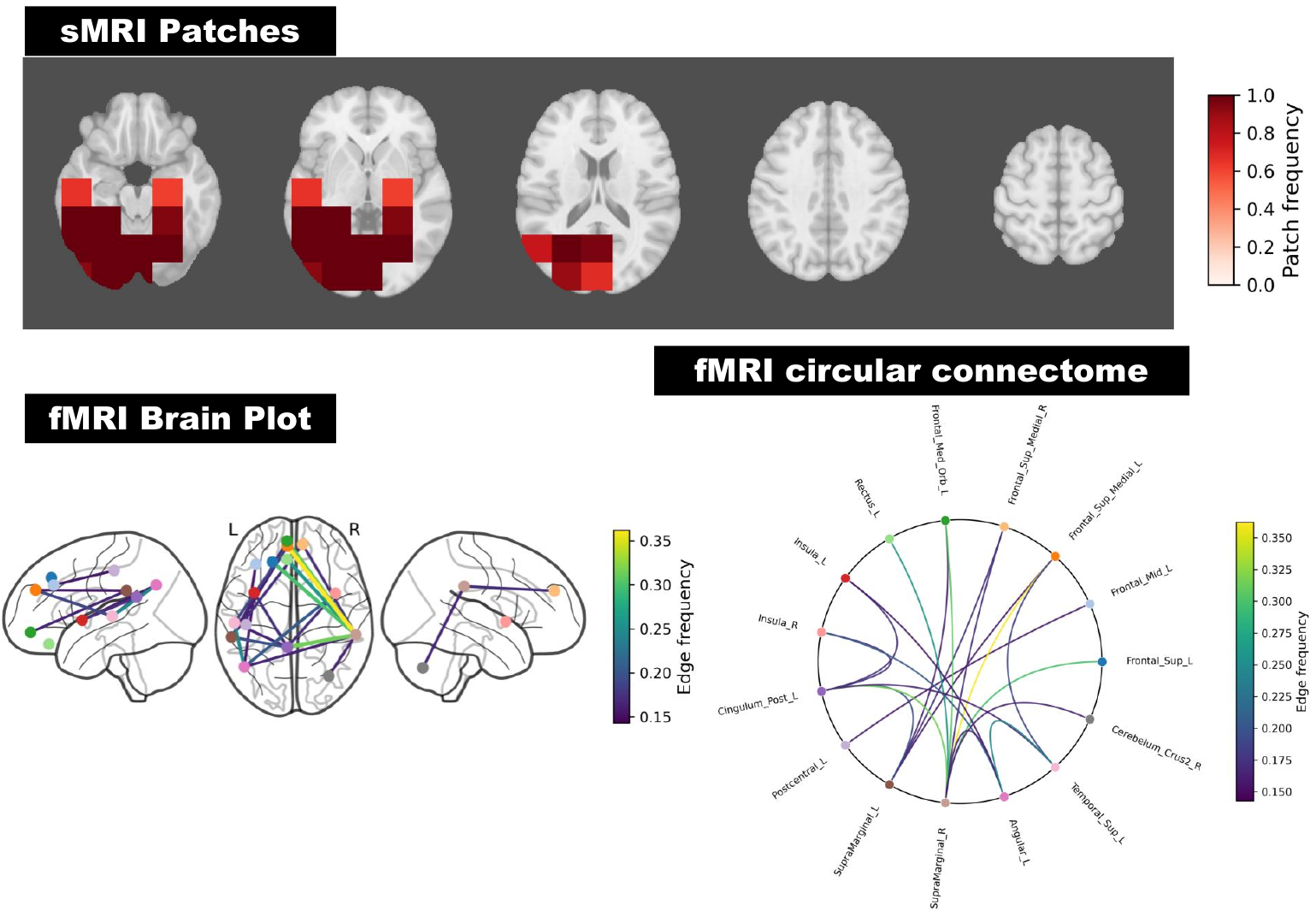}}
   \caption{\textbf{REST-meta-MDD} biomarker visualizations on two modalities. (a) HC ($k = 25$); (b) MDD ($k = 25$). Top: sMRI subvolumes on axial slices. Bottom left: selected rs-fMRI edges on the cortical surface, nodes colored by AAL-116 ROI. Bottom right: circular connectome of the same edges with ROI labels. Color encodes selection frequency. Descriptive aggregates, not statistical comparisons; one run, whose posterior structural pattern recurred in 6 of 10 runs.}
   \label{fig:REST-meta}
\end{figure*}

\subsection{Comparison to Post-hoc Explanations}

A central claim of self-explainable modelling is that explanations optimized jointly with the predictor are more faithful than those attached post hoc. To ensure a fair comparison, we apply LIME and SHAP to Multimodal IB~\cite{mai2022multimodal} with late fusion (L-MIB), which uses the same encoders as SMILE; the two models differ only in the built-in explainer. We evaluate the image modality with two metrics, both lower-is-better. \emph{Fidelity} is the drop in predicted-class probability when only the regions marked as important are retained. \emph{Consistency} is maximum sensitivity~\cite{yeh2019fidelity}: the largest change in the explanation
under a small input perturbation. SMILE attains the best value on both (Table~\ref {tab:example}), with visual examples in Fig.~\ref{fig: post-hoc}.

\begin{figure}[t!]
  \centering   \includegraphics[width=0.94\linewidth]{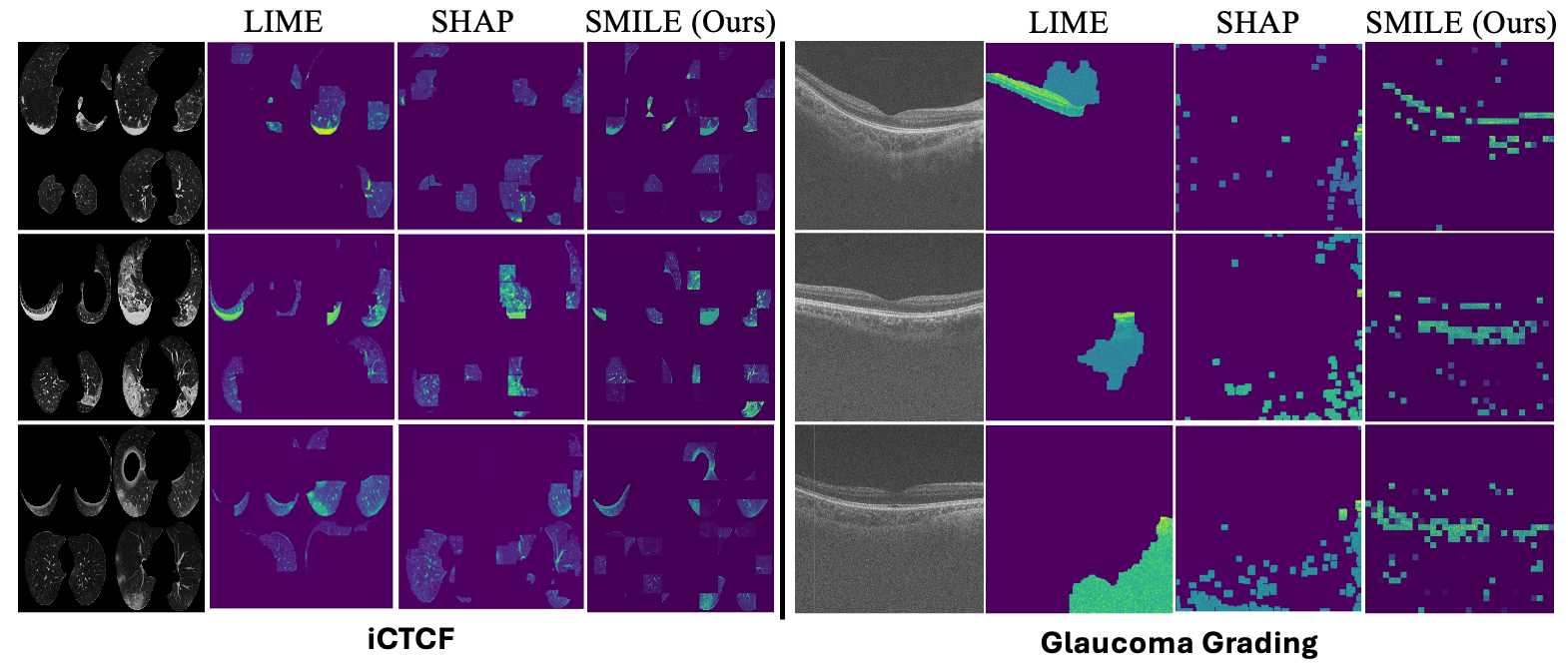}
   \caption{{Qualitative examples of explanations generated by LIME, SHAP, and SMILE on different datasets. For CT scans in \textbf{iCTCF} dataset, we select $k=30$ with the patch size of $70\times 70$; for OCT slices in \textbf{Glaucoma Grading} dataset, $k=100$ with the patch size of $4\times 4$.}}
   \label{fig: post-hoc}
\end{figure}

\begin{table}[t!]
    \centering
    \caption{Comparison between built-in (ours) and post-hoc explanations. A lower value indicates better performance for metrics.}
    \resizebox{0.7\linewidth}{!}{%
\renewcommand{\arraystretch}{1}
\begin{tabular}{ccccccc}
\toprule
\multirow{3}{*}{Dataset} &\multicolumn{3}{c}{\begin{tabular}[c]{@{}c@{}}Fidelity\\ (average confidence drop[\%] $\downarrow$ )\end{tabular}} & \multicolumn{3}{c}{\begin{tabular}[c]{@{}c@{}}Consistency\\ (maximum sensitivity$\downarrow$)\end{tabular}}  \\ \cmidrule(lr){2-4}\cmidrule(lr){5-7}
&LIME      & SHAP     & SMILE     & LIME     & SHAP     & SMILE     \\\hline
iCTCF & 36.4$\pm$2.1         &   48.1$\pm$1.8       &    \textbf{30.2$\pm$0.9}       &   0.48       & 0.37         &\textbf{0.22}\\
Glaucoma&12.0$\pm$0.8&15.1$\pm$1.4&\textbf{9.2$\pm$0.6}&0.39&0.41&\textbf{0.27}\\
\bottomrule
\end{tabular}}
    \label{tab:example}
\end{table}

\begin{table}[thp]
\centering
\caption{Ablation study results on different datasets. The modality mentioned refers to the object we operate on.}
\resizebox{0.7\linewidth}{!}{%
\renewcommand{\arraystretch}{1.1}
\begin{tabular}{c|c|cc}
\toprule
Dataset                      & \multicolumn{1}{c|}{Modality}      & w/o $X^m$ & w/ $E_r$  \\
\midrule
\multirow{8}{*}{\begin{tabular}[c]{@{}c@{}}BRCA/\\ ROSMAP\end{tabular}} 
& mRNA           & 80.4$\pm$1.1/78.6$\pm$1.8   &  84.7$\pm$1.9/82.3$\pm$2.0      \\
& DNA meth       & 84.2$\pm$0.7/81.2$\pm$1.4   & 85.0$\pm$1.5/81.8$\pm$1.8        \\
& miRNA          & 83.9$\pm$0.9/79.0$\pm$1.5    & 84.3$\pm$1.5/82.1$\pm$1.7        \\\cmidrule{2-4}
& mRNA+DNA meth  & 68.4$\pm$1.4/63.7$\pm$1.8    &  77.8$\pm$4.2/75.4$\pm$5.2        \\
& mRNA+miRNA     & 66.7$\pm$1.8/61.5$\pm$2.0   &  75.9$\pm$4.7/71.1$\pm$5.8       \\
& miRNA+DNA meth & 71.4$\pm$1.1/65.8$\pm$1.3    &  79.6$\pm$3.9/78.3$\pm$4.9        \\\cmidrule{2-4}
& All            & --    &  70.7$\pm$8.3/67.1$\pm$9.7  \\\cmidrule{2-4}
& \multicolumn{3}{c}{\textit{Ours}: \textbf{87.3$\pm$0.2/85.1$\pm$1.1}}\\ \midrule
\multirow{4}{*}{iCTCF}                                                      
& 2D Montage                &66.8$\pm$2.3        & 82.2$\pm$4.9         \\
& Clinical features         &68.9$\pm$2.7           &  83.6$\pm$ 3.6       \\\cmidrule{2-4}
& Both                         & --    & 80.3$\pm$ 7.2         \\ \cmidrule{2-4}
& \multicolumn{3}{c}{\textit{Ours}: \textbf{92.4$\pm$0.5}}\\ \midrule
\multirow{4}{*}{\begin{tabular}[c]{@{}c@{}}Glaucoma\\ Grading\end{tabular}} 
& Color fundus                 &64.3$\pm$2.1     &68.3$\pm$3.9          \\
& OCT slices                    &56.5$\pm$3.4     &65.9$\pm$4.8          \\\cmidrule{2-4}
& Both                         &--     & 67.1$\pm$6.2     \\ \cmidrule{2-4}
& \multicolumn{3}{c}{\textit{Ours}: \textbf{72.4$\pm$0.5}}\\ 
\midrule
\multirow{4}{*}{\begin{tabular}[c]{@{}c@{}}REST-meta-MDD\end{tabular}} 
& rs-fMRI               &  61.0$\pm$2.2   &   61.9 $\pm$1.0      \\
& sMRI                    & 52.5$\pm$1.1   &   56.6$\pm$6.6      \\\cmidrule{2-4}
& Both                         &--     &   63.9$\pm$0.9  \\ \cmidrule{2-4}
& \multicolumn{3}{c}{\textit{Ours}: \textbf{67.4$\pm$2.3}}\\ 
\bottomrule
\end{tabular}
}
\label{tab: ablation}
\end{table}

\subsection{Ablation Study}
In this section, we analyze the contributions of the explainer module to quantify the effectiveness of the generated explanations. We consider two variants of SMILE: (\romannumeral1) without specific modality $X^m$, which serves as a modality-ablation baseline. (\romannumeral2) with a random mask $E_r$, which modifies (\ref{eq: multi_loss2}) s.t. we use a random mask $r_i^m$ instead of explanation $e_i^m$ from the explainer. We present their performance with ACC scores (Kappa scores for the Glaucoma grading dataset) in Table~\ref{tab: ablation}.

The results reveal two key insights: (\romannumeral1) Modalities interact synergistically to achieve high performance, highlighting the need to assess each modality's contribution in multimodal contexts. (\romannumeral2) While random masks in information bottleneck training improve performance over the baseline, they introduce higher variance. In contrast, our explainer captures more decision-relevant information and exhibits greater stability than the random mask approach.

\section{Conclusion and Future Work}
\label{sec: conclude}

We proposed SMILE, a self-explainable multimodal framework grounded in the IB principle for AI-based medical diagnosis. SMILE comprises three modules: modality-specific explainers, encoders mapping medical data to low-dimensional latent representations, and a fused classifier.
Theoretically, under the single-modality setting and with $\beta=0$, the SMILE objective reduces to the widely used INVASE objective as a special case. Moreover, our analysis shows that introducing the IB-based compression term does not necessarily impose an intrinsic accuracy penalty; instead, it can improve the generalization bound by controlling the amount of information retained in the selected representations.
Across five representative multimodal medical datasets spanning heterogeneous modalities, SMILE exceeds state-of-the-art diagnostic performance while producing faithful modality-specific explanations.

Our future work is twofold. First, the sufficient encoder assumption~\cite{tian2020makes} in (\ref{eq: multi_loss2}) simplifies optimization but remains over-optimistic; estimating or approximating $I(X;\tilde{X})$ without it is an open problem, even from an information theory perspective. Second, SMILE assumes that all modalities may contribute and focuses on modality-specific informative features, leaving cross-modal structure implicit. Moving forward, we aim to explicitly model high-order modality interaction terms such as synergy.

\newpage

\bibliographystyle{unsrtnat}
\bibliography{main}
\newpage
\appendix

\section{Proof to Proposition~1}

Our proposed IB objective for single modality feature selection is expressed as:
\begin{equation}\label{eq: obj1_supp}
    \min_{E} - I(Y; \tilde{X}) + \beta I(\tilde{X}; {X}) + \lambda \|{E}\|_0.
\end{equation}

\begin{proposition} \label{proposition1}
    The IB objective in Eq.~(\ref{eq: obj1_supp}) encompasses that of INVASE as a special case when $\beta=0$, i.e., INVASE lacks a compression term. 
\end{proposition}

\begin{proof}
The INVASE formulates the learning of $E$ in a single modality as:
\begin{equation}\label{eq: obj_INVASE_supp}
    \min_E \mathbb{E}_{X} \left[ D_{\text{KL}}\left( p(Y|X) \| p(Y|\tilde{X}) \right) \right]  + \lambda \|{E}\|_0,
\end{equation}
where $D_{\text{KL}}$ refers to the Kullback-Leibler (KL) divergence.

The term $ \mathbb{E}_{X} \left[ D_{\text{KL}}\left( p(Y|X) \| p(Y|\tilde{X}) \right) \right] $ can be decomposed as follows:
\begin{equation}
\begin{split}
    & \mathbb{E}_{X} \left[ D_{\text{KL}}\left( p(Y|X) \| p(Y|\tilde{X}) \right) \right] \\
    & = \int_X p(X) \left( \int_Y p(Y|X) \log\left( \frac{p(Y|X)}{p(Y|\tilde{X})} \right) dY \right) dX \\
    & = \mathbb{E}_{p(X,Y)} \left[ \log\left( \frac{p(Y|X)}{p(Y|\tilde{X})} \right) \right] \\
    & = -H(Y|X) + H(Y|\tilde{X}) \\
    & = -H(Y|X) + H(Y) - (- H(Y|\tilde{X}) + H(Y)) \\
    & = I(X;Y) - I(Y;\tilde{X}),
\end{split}
\end{equation}
where we used the Markov chain property $\tilde{X}\leftarrow X\rightarrow Y$, i.e.,
$P(Y|X) = P(Y|X,\tilde{X})$.

From the above relationship, minimizing the expected KL divergence $ \mathbb{E}_{X} \left[ D_{\text{KL}}\left( p(Y|X) \| p(Y|\tilde{X}) \right) \right] $ is equivalent to maximizing the mutual information $I(Y;\tilde{X})$, since $I(X;Y)$ is a fixed value which only depends on training data and is irrelevant to optimization~\cite{piran2020dual}.

Hence, the INVASE objective essentially amounts to:
\begin{equation}
    \min_{E} - I(Y; \tilde{X}) + \lambda \|{E}\|_0.
\end{equation}    

\end{proof}

\section{Proof to Proposition~\ref{proposition2}}

\begin{proposition}\label{proposition2}
    With probability at least $1-\delta$ over the training data $t= \{ \overline{\overline{x}}_i,  y_i\}_{i=1}^N $ drawn from a data distribution $p(\overline{\overline{x}},y)$, where $ \overline{\overline{x}}_i = \{x_i^m\}_{m=1}^M$, the generalization error
    $\Delta(t) = \mathbb{E}_{p(\overline{\overline{x}},y)} \left[\ell(f^t(\overline{\overline{x}},y)) \right] - \frac{1}{N}\sum_{i=1}^N  \ell(f^t(\overline{\overline{x}}_i,y_i)) $ roughly obeys the following form:
\begin{equation}
    \Delta(t) = \tilde{\mathcal{O}} \left( \sqrt{ \frac{\sum_{m}^M I(\tilde{X}^m;X^m) +1}{N} } \right) \:\: \text{as} \:\: N \rightarrow \infty,
\end{equation}
    where $\ell$ is a bounded per-sample loss, $f^t$ is the full model obtained by training over $t$.
\end{proposition}

\begin{proof}

Our proof is based on Theorem~\ref{theorem}\cite{kawaguchi2023does}.

\begin{theorem}\label{theorem}
       With probability at least $1-\delta$ over the training data $t= \{ x_i,  y_i\}_{i=1}^N $ drawn from a data distribution $p(x,y)$, the generalization error
    $\Delta(t) = \mathbb{E}_{p(x,y)} \left[\ell(f^t(x,y)) \right] - \frac{1}{N}\sum_{i=1}^N  \ell(f^t(x_i,y_i)) $ roughly obeys the following form:
\begin{equation}
    \Delta(t) = \tilde{\mathcal{O}} \left( \sqrt{ \frac{ I(X;Z_l^t) +1}{N} } \right) \:\: \text{as} \:\: N \rightarrow \infty,
\end{equation}
    where $\ell$ is a bounded per-sample loss, $f^t$ is the full model obtained by training over $t$, $Z_l^t$ is representation obtained after passing $X$ through the first $l$ layers of model $f^t$.
\end{theorem}

In our case, we can treat $\{X^m\}_{i=1}^M$ as a joint information source, denoted by $\overline{\overline{X}}$, which transmits information through a parameterized channel to $\bar{Z}$ and subsequently to~$Y$ (see also Fig. 3 in the main text). Then, according to Theorem~\ref{theorem},
\begin{equation}
    \Delta(t) = \tilde{\mathcal{O}} \left( \sqrt{ \frac{ I(\overline{\overline{X}};\bar{Z}) +1}{N} } \right) \:\: \text{as} \:\: N \rightarrow \infty.
\end{equation}

Further, we can treat $\{Z^m\}_{i=1}^M$ as a joint information source, denoted by $\overline{\overline{Z}}$. Due to the data processing inequality ($\bar{Z}=g_\omega(\overline{\overline{Z}})$), the generalization error can be upper bounded by:
\begin{equation}
    \Delta(t) = \tilde{\mathcal{O}} \left( \sqrt{ \frac{ I(\overline{\overline{X}};\overline{\overline{Z}}) +1}{N} } \right) \:\: \text{as} \:\: N \rightarrow \infty.
\end{equation}

Lastly, we note that each $Z^m$ only depends on the corresponding $X^m$ through the modality-specific encoder, i.e.,
\begin{equation}
\begin{split}
    & X^1 \rightarrow Z^1, \\
    & X^2 \rightarrow Z^2, \\
    & \cdots \\
    & X^M \rightarrow Z^M. \\
\end{split}
\end{equation}

Then the conditional distribution of $ \overline{\overline{Z}} $ given $\overline{\overline{X}}$ factorized as:
\begin{equation}
    p( \overline{\overline{Z}}|\overline{\overline{X}} ) = \prod_{m=1}^M p(Z^m|X^m).
\end{equation}

So as long the channels are decoupled like this, we always have~\cite{polyanskiy2024information}\cite[Chapter~12]{wu2017lecture}:
\begin{equation}
    I(\overline{\overline{X}};\overline{\overline{Z}}) \leq \sum_{m=1}^M I(X^m;Z^m),
\end{equation}
which, due to the data processing inequality, can be further upper bounded by:
\begin{equation}
    I(\overline{\overline{X}};\overline{\overline{Z}}) \leq \sum_{m=1}^M I(X^m;\tilde{X}^m).
\end{equation}

Hence, the generalization error 
$\Delta(t)$ scales as $\tilde{\mathcal{O}} \left( \sqrt{ \frac{\sum_{m}^M I(\tilde{X}^m;X^m) +1}{N} } \right)$.

\end{proof}

\section{Implementation Details}

\subsubsection{BRCA}
\noindent\textbf{mRNA encoder.}
The mRNA encoder takes the selected 1000-dimensional expression vector as input and maps it to a 500-dimensional latent representation via a fully connected projection layer, followed by ReLU activation and dropout with $p=0.5$. The same encoder is applied to both the original and selected inputs, allowing the mutual-information regularizer to compare modality representations before and after selection. 

\textbf{mRNA selector.} The mRNA selector is an MLP with layer dimensions $1000\rightarrow2000\rightarrow2000\rightarrow1000$. GELU activations are applied after the first two linear layers, and the output layer produces one selection logit per mRNA feature. SMILE applies a differentiable top-$k$ relaxation to these logits and obtains the hard explanation mask by thresholding at the $k$-th largest logit. We set $k=30$ for this modality. 

\textbf{DNA-methylation encoder.} The DNA-methylation encoder follows the same fully connected design as the mRNA encoder. It maps the selected 1000-dimensional methylation vector to a 500-dimensional latent representation using a linear projection, ReLU activation, and dropout with $p=0.5$.

\textbf{DNA-methylation selector.} The DNA-methylation selector uses layer dimensions $1000\rightarrow2000\rightarrow2000\rightarrow1000$ with GELU activations between hidden layers. The selector produces feature-level logits over all methylation variables, and the SMILE top-$k$ operation retains $k=30$ selected methylation features. 

\textbf{miRNA encoder.}
The miRNA encoder maps the selected 503-dimensional miRNA-expression vector to a 500-dimensional latent representation. As in the other omics branches, the encoder consists of a fully connected projection followed by ReLU activation and dropout with $p=0.5$. 

\textbf{miRNA selector.} The miRNA selector uses layer dimensions $503\rightarrow1000\rightarrow1000\rightarrow503$. Its output logits are converted into a sparse differentiable mask by the SMILE selector, and the selected mask keeps $k=30$ miRNA features. 

\textbf{Classifier head.} The three selected modality representations are concatenated into a 1500-dimensional multimodal representation. A fully connected classifier maps this fused representation to the five PAM50 subtype labels. 

\textbf{Training details.} We train the BRCA model for 2000 epochs using Adam with an initial learning rate of $10^{-4}$ and weight decay $10^{-4}$. The learning rate is multiplied by 0.2 every 500 epochs. The mutual-information regularization weight is set to $\beta=0.08$.

\subsubsection{ROSMAP} \textbf{mRNA encoder.} The mRNA encoder takes the selected 200-dimensional expression vector and projects it to a 300-dimensional latent representation with a fully connected layer, ReLU activation, and dropout with $p=0.5$. 

\textbf{mRNA selector.} The mRNA selector is an MLP with layer dimensions $200\rightarrow400\rightarrow400\rightarrow200$. It outputs one selection logit per mRNA feature, and SMILE retains the top $k=30$ features through the differentiable top-$k$ selector. 

\textbf{DNA-methylation encoder.} The DNA-methylation encoder maps the selected 200-dimensional methylation vector to a 300-dimensional latent representation through the same fully connected encoder used for the mRNA branch. 

\textbf{DNA-methylation selector.} The methylation selector uses layer dimensions $200\rightarrow400\rightarrow400\rightarrow200$ with GELU activations. The selector produces methylation-feature logits and keeps $k=30$ methylation features.

\textbf{miRNA encoder.} The miRNA encoder maps the selected 200-dimensional miRNA-expression vector to a 300-dimensional latent representation using a fully connected projection, ReLU activation, and dropout with $p=0.5$. 

\textbf{miRNA selector.} The miRNA selector also uses layer dimensions $200\rightarrow400\rightarrow400\rightarrow200$. The selected miRNA mask is obtained with the same differentiable top-$k$ procedure, with $k=30$. 

\textbf{Classifier head.} The three selected 300-dimensional modality embeddings are concatenated into a 900-dimensional representation. A fully connected classifier maps the fused representation to the binary diagnostic label. 

\textbf{Training details.}
We train the ROSMAP model for 1000 epochs using Adam with an initial learning rate of $10^{-4}$ and weight decay $10^{-4}$. The learning rate is multiplied by 0.2 every 200 epochs. The mutual-information regularization weight is set to $\beta=0.1$.

\subsubsection{iCTCF} 
\textbf{Clinical-feature encoder.}
The clinical encoder maps the selected 81-dimensional clinical vector to a 1024-dimensional latent representation. It consists of two fully connected layers with dimensions $81\rightarrow64\rightarrow1024$, with ReLU activations after both layers. 

\textbf{Clinical-feature selector.} The clinical selector follows the non-image selector used for omics data. It uses an MLP with layer dimensions $81\rightarrow100\rightarrow100\rightarrow81$ and GELU activations between hidden layers. The output logits define feature-level selection scores over all clinical variables, and SMILE keeps $k=20$ selected clinical features. 

\textbf{HRCT encoder.} The HRCT encoder takes a ten-channel $700\times700$ image tensor as input. It uses three convolutional blocks. The first block maps the input channels to 32 channels with a $3\times3$ convolution of stride 2, followed by ReLU and $2\times2$ max pooling. The second and third blocks use the same design and map $32\rightarrow64$ and $64\rightarrow128$ channels, respectively. The spatial resolution is reduced as $700\times700\rightarrow175\times175\rightarrow44\times44\rightarrow11\times11$. The resulting feature map is flattened and projected to a 1024-dimensional HRCT representation, followed by ReLU activation. 

\textbf{HRCT selector.} The HRCT selector operates at the patch level. It applies a lightweight convolutional selector to produce logits on a $10\times10$ patch grid for each montage channel, corresponding to $70\times70$ image patches in the original $700\times700$ resolution. The selector logits are converted into a differentiable top-$k$ mask, which is upsampled by nearest-neighbor interpolation before being multiplied by the HRCT input. We set the HRCT selection budget to $k=60$ patches. 

\textbf{Classifier head.} The 1024-dimensional clinical representation and 1024-dimensional HRCT representation are concatenated into a 2048-dimensional multimodal vector. The classifier maps this vector through $2048\rightarrow256\rightarrow2$ with ReLU activation before the output layer. 

\textbf{Training details.} We train five models from scratch, one for each validation fold, for 300 epochs using Adam with a learning rate of $10^{-3}$, weight decay of $10^{-4}$, and a batch size of 8. The model with the highest validation AUC is selected for final testing. The mutual-information regularization weights are set to 0.01 for the clinical branch and 0.005 for the HRCT branch.

\subsubsection{Glaucoma Grading} \textbf{Fundus encoder.} The fundus encoder uses an EfficientNet-B3 branch to map the selected color fundus image to a 1000-dimensional representation. This representation is used as the modality embedding for multimodal fusion. 

\textbf{Fundus selector.}
The fundus selector operates on the resized $256\times256$ RGB image. It contains three convolutional blocks with channel progression $3\rightarrow64\rightarrow128\rightarrow256$. Each block uses a $3\times3$ convolution, batch normalization, ReLU activation, and $2\times2$ max pooling, reducing the spatial grid to $32\times32$. A final $1\times1$ convolution produces patch-level logits. SMILE selects $k=100$ fundus patches from this grid, upsamples the mask to the input resolution, and multiplies it with the fundus image before encoding. This corresponds to $8\times8$ patches in the original fundus resolution. 

\textbf{OCT encoder.} The OCT encoder uses a ResNet-18 branch adapted to ten-channel OCT input. Specifically, the first convolution is modified to accept 10 input channels, and the final fully connected classification layer is removed. The resulting OCT representation is 512-dimensional. 

\textbf{OCT selector.} The OCT selector contains four convolutional blocks with channel progression $10\rightarrow64\rightarrow128\rightarrow256\rightarrow512$. Each block uses a $3\times3$ convolution, batch normalization, ReLU activation, and $2\times2$ max pooling. The selector therefore reduces the $512\times512$ OCT tensor to a $32\times32$ patch-logit grid. A final $1\times1$ convolution outputs OCT selection logits, and SMILE keeps $k=100$ OCT patches. The selected mask is upsampled to the OCT input resolution and applied before the OCT encoder. Each selected location corresponds to a $16\times16$ OCT patch. 

\textbf{Classifier head.} The 1000-dimensional fundus representation and 512-dimensional OCT representation are concatenated into a 1512-dimensional multimodal vector. The classifier maps this vector through $1512\rightarrow756\rightarrow3$ with ReLU activation before the final three-class output. 

\textbf{Training details.} We train with Adam using a learning rate of $10^ {-3}$, weight decay of $10^ {-4}$, batch size 8, and a maximum of 1000 epochs. A step scheduler multiplies the learning rate by 0.2 every 500 epochs. The mutual-information regularization weight is set to 0.001 for each image modality.

\subsubsection{REST-meta-MDD}

\textbf{Functional connectivity graphs.}
For rs-fMRI, we use ROI-averaged BOLD time series parcellated by AAL-116. The preprocessing follows the consortium pipeline: the first 10 volumes are discarded; slice-timing correction, head-motion realignment, MNI normalization, temporal band-pass filtering from 0.01 to 0.10 Hz, and nuisance regression are applied. The nuisance regressors include head motion, global brain signal, white matter, cerebrospinal fluid, and linear/quadratic drift terms. For each subject, we compute a $116\times116$ Pearson functional-connectivity matrix and apply Fisher's $z$ transform. Nodes correspond to AAL ROIs, and each node feature is its full Fisher-$z$ correlation profile. To obtain a shared topology, subject-level connectivity matrices are averaged at the dataset level, and the top 20\% positive correlations are retained as admissible edges. This gives a fixed binary adjacency shared by all subjects. 

\textbf{Graph encoder.}
The rs-fMRI encoder is a 3-layer GCN with ReLU activations and dropout $p=0.15$. It maps node features through $116\rightarrow116\rightarrow96\rightarrow64$. The resulting 64-dimensional node embeddings are $\ell_2$-normalized and aggregated by second-order pooling. Specifically, a 3-layer MLP maps node embeddings through $64\rightarrow32\rightarrow32\rightarrow32$, and the bilinear outer product is flattened into a 1024-dimensional graph embedding. 

\textbf{Subgraph selector.} The subgraph selector first linearly projects ROI correlation profiles to 64-dimensional node embeddings. It then constructs pairwise edge embeddings by concatenating node pairs, maps each 128-dimensional edge embedding through an MLP $128\rightarrow64\rightarrow1$, and reshapes the outputs into an edge-logit matrix. Non-admissible edges are masked out using the shared adjacency. A soft top-$k$ operation converts the remaining logits into a sparse differentiable edge mask. The selection temperature is annealed from 2.0 to 0.05 during warm-up to avoid premature hard selection while preserving gradient flow. In our multimodal experiment, the selector retains $k=25$ rs-fMRI edges, approximately 1\% of admissible pairs. 

\textbf{sMRI preprocessing.}
For sMRI, we use gray-matter-volume (GMV) maps generated by a voxel-based morphometry pipeline. The pipeline includes bias-field correction, GM/WM/CSF segmentation, MNI normalization, and cross-site intensity normalization. The input image is the mwc1 GMV map with spatial size $121\times145\times121$, and each subject's GMV map is $z$-normalized before model input.

\textbf{sMRI encoder.} The sMRI encoder is a five-layer 3D CNN adapted to volumetric inputs. The first four convolutional layers use kernel size 3, stride 1, padding 1, batch normalization, ReLU activation, and $2\times2\times2$ max pooling, reducing the spatial size as
$121\times145\times121 \rightarrow 60\times72\times60 \rightarrow 30\times36\times30 \rightarrow 15\times18\times15 \rightarrow 7\times9\times7$.
The fifth convolution maps the feature map to 64 channels without further pooling, preserving the $7\times9\times7$ grid. The channel progression is $28\rightarrow58\rightarrow128\rightarrow256\rightarrow64$. The final feature map is globally averaged and passed through a projection head with dropout $p=0.5$ and two GELU-activated fully connected layers $64\rightarrow64\rightarrow64$, producing a 64-dimensional subject embedding. 

\begin{figure*}[!t]
\centering
\subfloat[]{
    \includegraphics[width=0.48\textwidth]{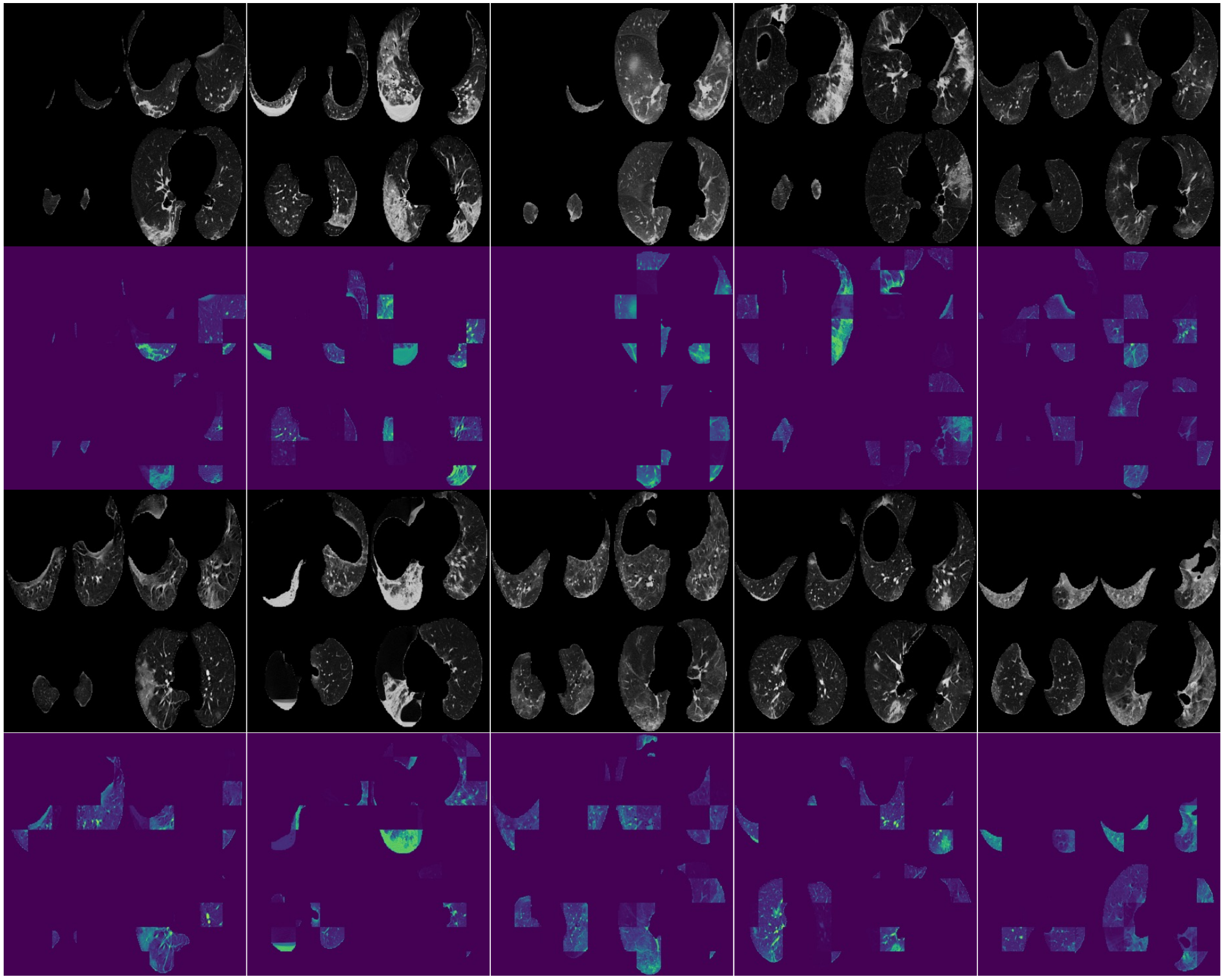}
}
\vspace{1mm}
\subfloat[]{
    \includegraphics[width=0.48\textwidth]{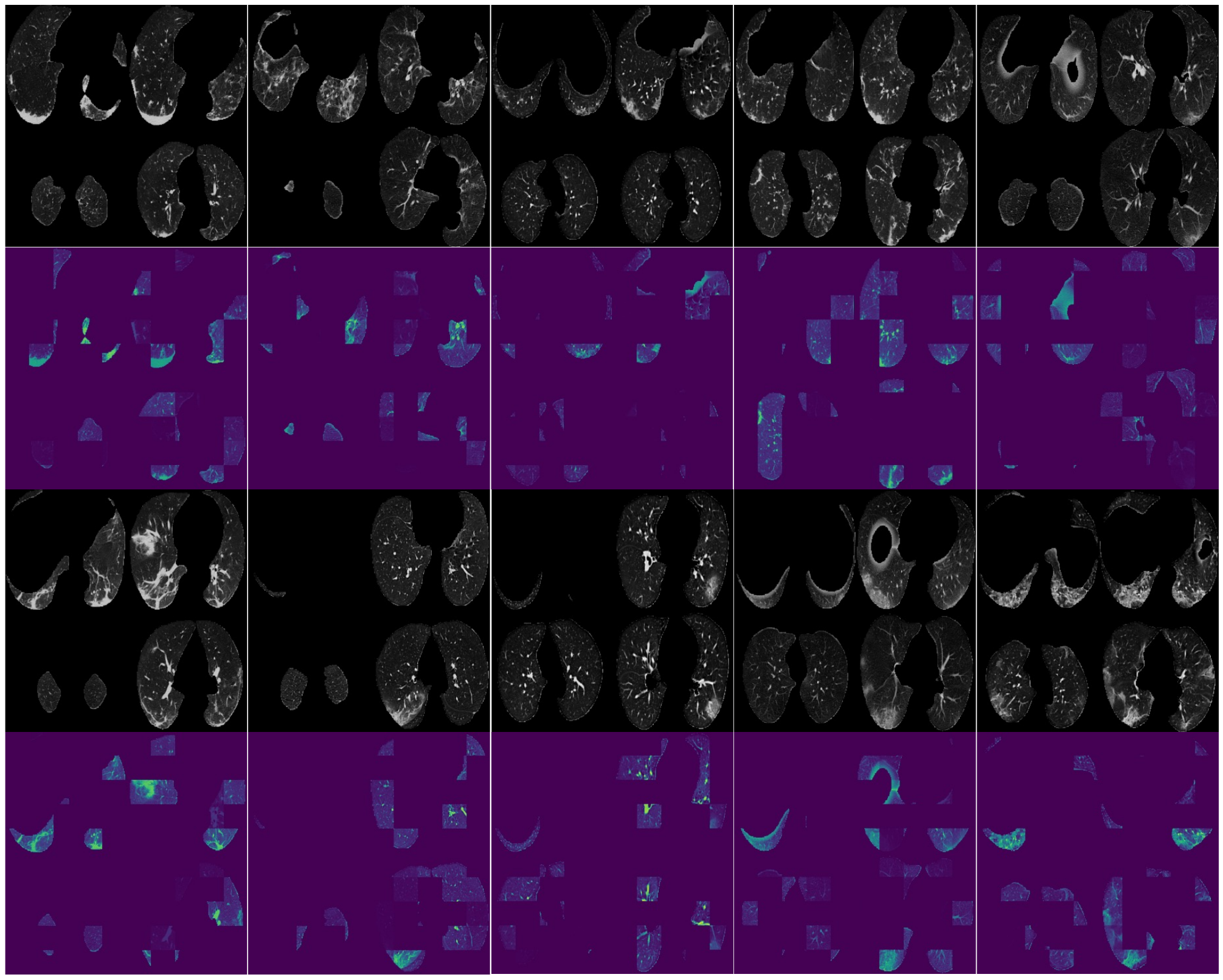}
}
\caption{\textbf{iCTCF} visualizations on 2D HRCT montages. Panels (a) and (b) show representative severe and mild patients, respectively. The highlighted regions indicate the HRCT patches selected by SMILE as prognostic evidence.}
\label{fig: lung}
\end{figure*}

\begin{figure*}[!t]
\centering
\subfloat[]{
    \includegraphics[width=0.48\textwidth]{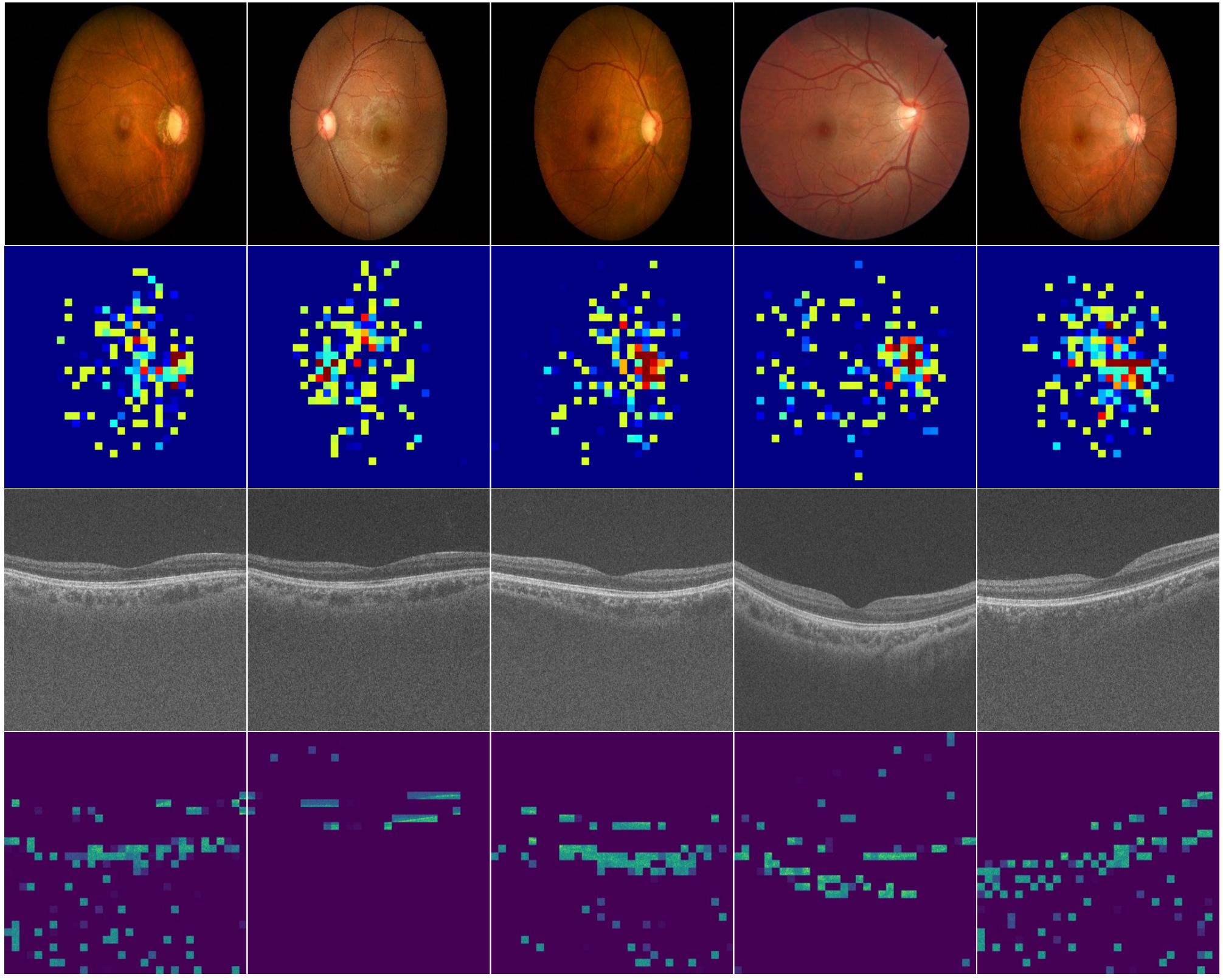}
}
\vspace{1mm}
\subfloat[]{
    \includegraphics[width=0.48\textwidth]{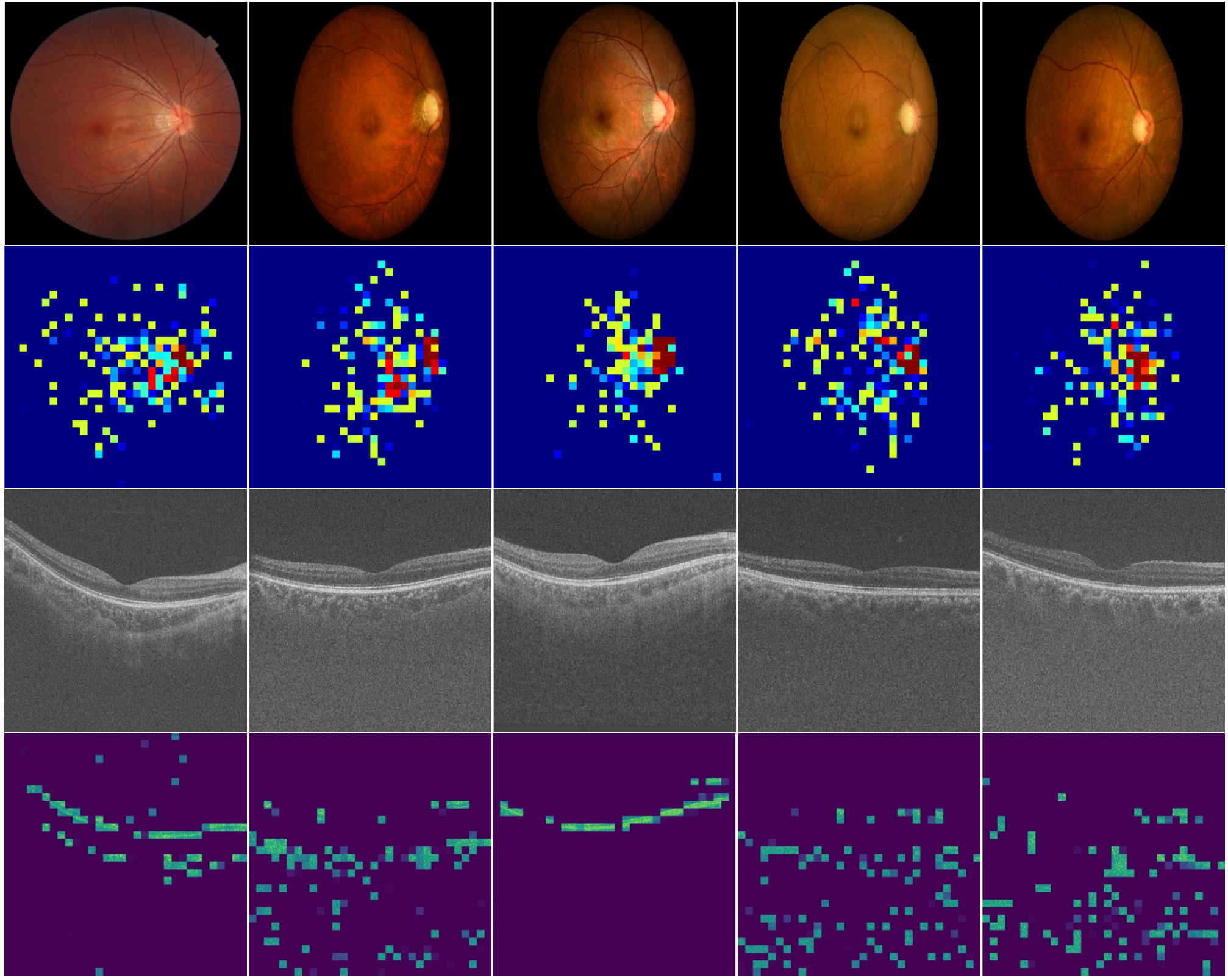}
}
\caption{\textbf{Glaucoma Grading} visualizations. Panels (a) and (b) show additional representative cases under the same multimodal setting. SMILE highlights selected regions on color fundus images and OCT slices, providing modality-specific evidence for glaucoma-grade prediction.}
\label{fig: glau}
\end{figure*}

\textbf{Volume selector.}
The sMRI selector is a lightweight 3D convolutional network that reduces the volume to a $7\times9\times7$ logit grid, corresponding to 441 candidate subvolumes. A final $1\times1\times1$ convolution outputs one logit per subvolume. Soft top-$k$ selection retains $k=25$ subvolumes, approximately 5\% of the grid. During the first 200 warm-up epochs, the selection budget is annealed from 441 to 25 and the temperature is annealed from 1.0 to 0.05. Logistic noise is added to the subvolume logits for robustness. The resulting soft mask is upsampled to the input resolution and multiplied with the GMV volume before encoding. 

\textbf{Classifier head.}
The 1024-dimensional rs-fMRI graph embedding and 64-dimensional sMRI subject embedding are concatenated into a 1088-dimensional multimodal representation and passed to a classifier head for MDD diagnosis. 

\textbf{Training details.} The maximum training budget is 2000 epochs. The encoders are warmed up for 200 epochs before joint optimization, and early stopping is applied with a patience of 100 epochs according to validation performance. The batch size is 32. Experiments are repeated over 10 runs. The first random seed is 42, and each subsequent run increases the seed by 1024.

\section{Additional Visualization Results}
\noindent\textbf{iCTCF visualizations.}
For iCTCF, we visualize both modalities (see Fig.~\ref{fig: lung}). The HRCT explanation highlights the selected patches on the 2D montage, enabling spatial inspection of imaging evidence. The clinical-feature explanation reports the normalized importance of the selected clinical variables over the test set. We separately show representative severe and mild cases so that the reader can assess whether the model attends to different image regions across prognostic outcomes.


\noindent\textbf{Glaucoma Grading visualizations.}
For Glaucoma Grading, we visualize the color fundus and OCT modalities separately. The selected regions are displayed as heatmaps, enabling examination of whether the model uses anatomically meaningful evidence in each image stream (see Fig.~\ref{fig: glau}).

\end{document}